\def\usearxivstyle{1}

\documentclass[11pt]{article}
\usepackage{fullpage}
\makeatletter
\long\def\@makecaption#1#2{
  \vskip 0.8ex
  \setbox\@tempboxa\hbox{\small {\bf #1:} #2}
  \parindent 1.5em  %% How can we use the global value of this???
  \dimen0=\hsize
  \advance\dimen0 by -3em
  \ifdim \wd\@tempboxa >\dimen0
  \hbox to \hsize{
    \parindent 0em
    \hfil 
    \parbox{\dimen0}{\def\baselinestretch{0.96}\small
      {\bf #1.} #2
    } 
    \hfil}
  \else \hbox to \hsize{\hfil \box\@tempboxa \hfil}
  \fi
}
\makeatother

\usepackage{amsfonts}
\usepackage{amsmath}
\usepackage{amsthm}
\usepackage{amssymb}
\usepackage{caption}
\usepackage{microtype}
\usepackage{graphicx}
\usepackage{subcaption}
\usepackage{booktabs} % for professional tables
\usepackage[numbers]{natbib}
\usepackage{fancyhdr}
\usepackage{color}
\usepackage{algorithm}
\usepackage{algorithmic}
\usepackage{forloop}

\DeclareMathOperator*{\argmin}{arg\,min}

\newtheorem{theorem}{Theorem}[section]
\newtheorem{lemma}[theorem]{Lemma}

\newtheorem{example}[theorem]{Example}
\newtheorem{proposition}[theorem]{Proposition}
\newtheorem{corollary}[theorem]{Corollary}

\newcommand{\tengyu}[1]{}
\newcommand{\tnote}[1]{}
\newcommand{\pl}[1]{}
\newcommand{\ak}[1]{}

\usepackage{hyperref}

\begin{document}

% Control whitespace around equations
\abovedisplayskip=8pt plus0pt minus3pt
\belowdisplayskip=8pt plus0pt minus3pt

\begin{center}
  {\LARGE Understanding Self-Training for Gradual Domain Adaptation} \\
  \vspace{.5cm}
  {\Large Ananya Kumar ~~~~ Tengyu Ma ~~~~ Percy Liang} \\
  \vspace{.2cm}
  {\large Stanford University} \\
  % \vspace{.1cm}
  Department of Computer Science \\
  \vspace{.2cm}
  \texttt{\{ananya,tengyuma,pliang\}@cs.stanford.edu}
\end{center}

\begin{abstract}
\noindent Machine learning systems must adapt to data distributions that evolve over time, in applications ranging from sensor networks and self-driving car perception modules to brain-machine interfaces. We consider gradual domain adaptation, where the goal is to adapt an initial classifier trained on a source domain given only unlabeled data that shifts gradually in distribution towards a target domain. We prove the first non-vacuous upper bound on the error of self-training with gradual shifts, under settings where directly adapting to the target domain can result in unbounded error. The theoretical analysis leads to algorithmic insights, highlighting that regularization and label sharpening are essential even when we have infinite data, and suggesting that self-training works particularly well for shifts with small Wasserstein-infinity distance. Leveraging the gradual shift structure leads to higher accuracies on a rotating MNIST dataset and a realistic Portraits dataset.
\end{abstract}

\section{Introduction}
\label{introduction}

Machine learning models are typically trained and tested on the same data distribution.
However, when a model is deployed in the real world, the data distribution typically evolves over time, leading to a drop in performance.
This problem is widespread: sensor measurements drift over time due to sensor aging~\cite{vergara2012Chemical}, self-driving car vision modules have to deal with evolving road conditions~\cite{bobu2018adapting}, and neural signals received by brain-machine interfaces change within the span of a day~\cite{farshchian2019adversarial}.
Repeatedly gathering large sets of labeled examples to retrain the model can be impractical, so we would like to leverage unlabeled examples to adapt the model to maintain high accuracy~\cite{farshchian2019adversarial, sethi2017reliable}.

In these examples the domain shift doesn't happen at one time, but happens gradually, although this gradual structure is ignored by most domain adaptation methods.
Intuitively, it is easier to handle smaller shifts, but for each shift we can incur some error so the more steps, the more degradation---making it unclear whether leveraging the gradual shift structure is better than directly adapting to the target.

In this paper, \emph{we provide the first theoretical analysis showing that gradual domain adaptation provides improvements over the traditional approach of direct domain adaptation}. We analyze self-training (also known as pseudolabeling), a method in the semi-supervised learning literature~\cite{chapelle2006semisupervised} that has led to state-of-the-art results on ImageNet~\cite{xie2020selftraining} and adversarial robustness on CIFAR-10~\cite{uesato2019are, carmon2019unlabeled, najafi2019robustness}. 

\begin{figure}[t]
\begin{center}
\ifdefined\usearxivstyle
\centerline{\includegraphics[width=0.6\columnwidth]{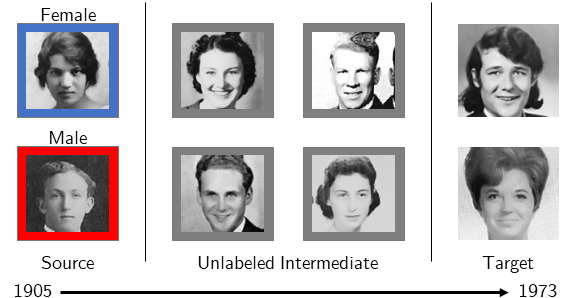}}
\else
\centerline{\includegraphics[width=\columnwidth]{images/portraits_example.png}}
\fi
\caption{In gradual domain adaptation we are given labeled data from a source domain, and unlabeled data from intermediate domains that shift gradually in distribution towards a target domain. Here, blue = female, red = male, and gray = unlabeled data.}
\label{fig:portraits_example}
\end{center}
\vskip -0.3in
\end{figure}

As a concrete example of our setting, the Portraits dataset~\cite{ginosar2017portraits} contains photos of high school seniors taken across many years, labeled by gender (Figure~\ref{fig:portraits_example}).
We use the first 2000 images (1905 - 1935) as the source, next 14000 (1935 - 1969) as intermediate domains, and next 2000 images as the target (1969 - 1973). A model trained on labeled examples from the source gets 98\% accuracy on held out examples in the same years, but only 75\% accuracy on the target domain. Assuming access to \emph{unlabeled} images from intermediate domains, our goal is to adapt the model to do well on the target domain. Direct adaptation to the target with self-training only improves the accuracy a little, from 75\% to 77\%.

\pl{make it clearer that we're proposing this or it has been studied (in which case we need a cite)}
The gradual self-training algorithm begins with a classifier $w_0$ trained on labeled examples from the source domain (Figure~\ref{fig:intro_timestep0}). For each successive domain $P_t$, the algorithm generates pseudolabels for unlabeled examples from that domain, and then trains a regularized supervised classifier on the pseudolabeled examples. The intuition, visualized in Figure~\ref{fig:grad_st_intuition}, is that after a single gradual shift, most examples are pseudolabeled correctly so self-training learns a good classifier on the shifted data, but the shift from the source to the target can be too large for self-training to correct. \emph{We find that gradual self-training on the Portraits dataset improves upon direct target adaptation (77\% to 84\% accuracy)}.

\begin{figure}

     \begin{center}
     \hfill
     \ifdefined\usearxivstyle
     \begin{subfigure}[b]{0.14\columnwidth}
     \else
     \begin{subfigure}[b]{0.24\columnwidth}
     \fi
         \centering
         \includegraphics[width=\textwidth]{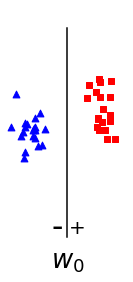}
         \caption{t = 0}
         \label{fig:intro_timestep0}
     \end{subfigure}
     \hfill
     \ifdefined\usearxivstyle
     \begin{subfigure}[b]{0.14\columnwidth}
     \else
     \begin{subfigure}[b]{0.24\columnwidth}
     \fi
         \centering
         \includegraphics[width=\textwidth]{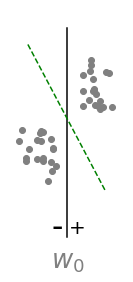}
         \caption{t = 1}
         \label{fig:intro_timestep1}
     \end{subfigure}
     \hfill
     \ifdefined\usearxivstyle
     \begin{subfigure}[b]{0.14\columnwidth}
     \else
     \begin{subfigure}[b]{0.24\columnwidth}
     \fi
         \centering
         \includegraphics[width=\textwidth]{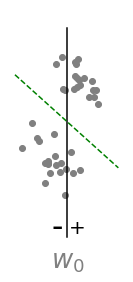}
         \caption{t = 2}
         \label{fig:intro_timestep2}
     \end{subfigure}
     \hfill
     \ifdefined\usearxivstyle
     \begin{subfigure}[b]{0.14\columnwidth}
     \else
     \begin{subfigure}[b]{0.24\columnwidth}
     \fi
         \centering
         \includegraphics[width=\textwidth]{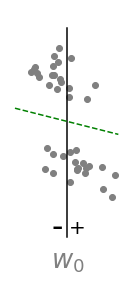}
         \caption{t = 3}
         \label{fig:intro_timestep3}
     \end{subfigure}
     \hfill
     \hfill
     \caption{
     % A simple example where the data distribution shifts gradually over time.
     The source classifier $w_0$ gets 100\% accuracy on the source domain (Figure~\ref{fig:intro_timestep0}), where we have labeled data.
     But after 3 time steps (Figure~\ref{fig:intro_timestep3}) the source classifier is stale, classifying most examples incorrectly.
     Now, we cannot correct the classifier using unlabeled data from the target domain, which corresponds to traditional domain adaptation directly to the target.
     Given \emph{unlabeled} data in an intermediate domain (Figure~\ref{fig:intro_timestep1}) where the shift is gradual, the source classifier pseudolabels most points correctly, and self-training learns an accurate classifier (show in green) that separates the classes.
     Successively applying self-training learns a good classifier on the target domain (green classifier in Figure~\ref{fig:intro_timestep3}).
     }
	\label{fig:grad_st_intuition}
	\end{center}
	\vskip -0.2in
\end{figure}

\textbf{Our results}: We analyze gradual domain adaptation in two settings.
The key challenge for domain adaptation theory is dealing with source and target domains whose support do not overlap~\cite{zhao2019zhao, shu2018dirtt}, which are typical in the modern high-dimensional regime.
The gradual shift structure inherent in many applications provides us with leverage to handle adapting to target distributions with non-overlapping support.

Our first setting, the margin setting, is distribution-free---we only assume that at every point in time there exists some linear classifier that can classify most of the data correctly with a margin, where the linear classifier may be different at each time step (so this is more general than covariate shift), and that the shifts are small in Wasserstein-infinity distance.
A simple example (as in Figure~\ref{fig:grad_st_intuition}) shows that a classifier that gets 100\% accuracy can get 0\% accuracy after a constant number of time steps.
Directly adapting to the final target domain also gets 0\% accuracy.
Gradual self-training does better, letting us bound the error after $T$ steps: $\mbox{err}_T \le e^{cT} (\alpha_0 + O(1/\sqrt{n}))$, where $\alpha_0$ is the error of the classifier on the source domain, and $n$ is the number of unlabeled examples in each intermediate domain.
While this bound is exponential in $T$, this bound is non-vacuous for small $\alpha_0$, and we show that this bound is tight for gradual self-training.

In the second setting, stronger distributional assumptions allow us to do better---we assume that $P(X \mid Y = y)$ is a $d$-dimensional isotropic Gaussian for each $y$. Here, we show that if we begin with a classifier $w_0$ that is nearly Bayes optimal for the initial distribution, we can recover a classifier $w_T$ that is Bayes optimal for the target distribution with infinite unlabeled data.
This is an idealized setting to understand what properties of the data might allow self-training to do better than the exponential bound.

\emph{Our theory leads to practical insights, showing that regularization---even in the context of infinite data---and label sharpening are essential} for gradual self-training. Without regularization, the accuracy of gradual self-training drops from 84\% to 77\% on Portraits and 88\% to 46\% on rotating MNIST. Even when we self-train with more examples, the performance gap between regularized and unregularized models stays the same---unlike in supervised learning where the benefit of regularization diminishes as we get more examples.
\pl{I was expecting some results for label sharpening; people might not even know what label sharpening is...}
\ak{Ah I see, it seems like a slightly less major point, and hard to define---what do you think?}
\pl{just thinking about it, label sharpening is a bit of a misnomer - we're sharpening the distribution over labels, not the labels,
but I don't have any other better ideas}

Finally, our theory suggests that the gradual shift structure helps when the shift is small in Wasserstein-infinity distance as opposed to other distance metrics like the KL-divergence.
For example, one way to interpolate between the source and target domains is to gradually introduce more images from the target, but this shift is large in Wasserstein-infinity distance---we see experimentally that gradual self-training does not help in this setting.
We hope this gives practitioners some insight into when gradual self-training can work.

\section{Setup}
\label{setup}

\newcommand{\dist}{\ensuremath{\rho}}
\newcommand{\wasser}{\ensuremath{W_{\infty}}}
\newcommand{\xspace}{\ensuremath{\mathbb{R}^d}}
\newcommand{\yspace}{\ensuremath{\{-1, 1\}}}
\newcommand{\error}{\ensuremath{\mbox{Err}}}
\newcommand{\E}[0]{\ensuremath{\mathop{\mathbb{E}}}}
\newcommand{\prob}[0]{\ensuremath{\mathbb{P}}}
\newcommand{\sign}{\ensuremath{\mbox{sign}}}
\newcommand{\selftrain}{\ensuremath{\textup{ST}}}
\newcommand{\linearmodels}{\ensuremath{\Theta}}

\newcommand{\marginphi}{\ensuremath{\phi}}
\newcommand{\ramp}{\ensuremath{r}}
\newcommand{\hinge}{\ensuremath{h}}
\newcommand{\rampl}{\ensuremath{\ell_r}}
\newcommand{\hingel}{\ensuremath{\ell_h}}
\newcommand{\phil}{\ensuremath{\ell_{\phi}}}
\newcommand{\rampL}{\ensuremath{L_r}}
\newcommand{\hingeL}{\ensuremath{L_h}}
\newcommand{\phiL}{\ensuremath{L_{\phi}}}
\newcommand{\reg}{\ensuremath{R}}
\newcommand{\normal}{\ensuremath{\mathcal{N}}}

\textbf{Gradually shifting distributions:} Consider a binary classification task of predicting labels $y \in \yspace$ from input features $x \in \xspace$. 
We have joint distributions over the inputs and labels, $\xspace \times \yspace$: $P_0, P_1, \ldots, P_T$, where $P_0$ is the source domain, $P_T$ is the target domain, and $P_1, \ldots, P_{T-1}$ are intermediate domains.
We assume the shift is gradual: for some $\epsilon > 0$, $\dist(P_t, P_{t+1}) < \epsilon$ for all $0 \leq t < T$, where $\dist(P, Q)$ is some distance function between distributions $P$ and $Q$.
We have $n_0$ labeled examples $S_0 = \{ x^{(0)}_i, y^{(0)}_i \}_{i=1}^{n_0}$ sampled independently from the source $P_0$ and $n$ unlabeled examples $S_t = \{ x^{(t)}_i \}_{i=1}^{n}$ sampled independently from $P_t$ for each $1 \leq t \leq T$.

\textbf{Models and objectives:}
We have a model family $\Theta$, where a model $M_{\theta}: \xspace \to \mathbb{R}$ outputs a score representing its confidence that the label $y$ is 1 for the given example.
The model's prediction for an input $x$ is $\sign(M_{\theta}(x))$, where $\sign(r) = 1$ if $r \geq 0$ and $\sign(r) = -1$ if $r < 0$.
We evaluate models on the fraction of times they make a wrong prediction, also known as the $0$-$1$ loss:
\begin{equation} \error(\theta, P) = \E_{X, Y \sim P} [ \sign(M_{\theta}(X)) \neq Y ] \end{equation}
The goal is to find a classifier $\theta$ that gets high accuracy on the target domain $P_T$---that is, low $\error(\theta, P_T)$.
In an online setting we may care about the accuracy at the current $P_t$ for every time $t$, and our analysis works in this setting as well.

\textbf{Baseline methods:}
We select a loss function $\ell : \mathbb{R} \times \yspace \to \mathbb{R}^+$ which takes a prediction and label, and outputs a non-negative loss value,
and we begin by training a source model $\theta_0$ that minimizes the loss on labeled data in the source domain:
\begin{equation} \theta_0 = \argmin_{\theta' \in \Theta} \frac{1}{n_0} \sum_{(x_i, y_i) \in S_0} \ell( M_{\theta'}(x_i), y_i ) \end{equation}

The \emph{non-adaptive baseline} is to use $\theta_0$ on the target domain, which incurs error $\error(\theta_0, P_T)$.
\emph{Self-training} uses unlabeled data to adapt a model.
Given a model $\theta$ and unlabeled data $S$, $\selftrain(\theta, S)$ denotes the output of self-training.
Self-training pseudolabels each example in $S$ using $M_{\theta}$, and then selects a new model $\theta'$ that minimizes the loss on this pseudolabeled dataset.
Formally,
\begin{equation} \label{eqn:selfTrainSample} \selftrain(\theta, S) = \argmin_{\theta' \in \Theta} \frac{1}{|S|} \sum_{x_i \in S} \ell( M_{\theta'}(x_i), \sign(M_{\theta}(x_i)) ) \end{equation}

Here, self-training uses ``hard" labels: we pseudolabel examples as either $-1$ or $1$, based on the output of the classifier, instead of a probabilistic label based on the model's confidence---we refer to this as \emph{label sharpening}.
In our theoretical analysis, we sometimes want to describe the behavior of self-training when run on infinite unlabeled data from a probability distribution $P$:
\begin{equation} \label{eqn:selfTrainPop} \selftrain(\theta, P) = \argmin_{\theta' \in \Theta} \E_{X \sim P} [ \ell( M_{\theta'}(X), \sign(M_{\theta}(X)) ) ] \end{equation}

The \emph{direct adaptation to target} baseline takes the source model $\theta_0$ and self-trains on the target data $S_T$, and is denoted by $\selftrain(\theta_0, S_T)$. Prior work often chooses to repeat this process of self-training on the target $k$ times, which we denote by $\selftrain_k(\theta_0, S_T)$.

\textbf{Gradual self-training:} In gradual self-training, we self-train on the finite unlabeled examples from each domain successively. That is, for $i \geq 1$, we set:
\begin{equation} \theta_{i} = \selftrain(\theta_{i-1}, S_i) \end{equation}
$\selftrain(\theta_0, (S_1, \ldots, S_T)) = \theta_T$ is the output of gradual self-training, which we evaluate on the target distribution $P_T$.

\pl{should we admit that GST actually uses more data, but that we can adjust? I'm on the fence about not complicating things but also pre-empting readers' worries}
\ak{Would prefer not complicating things here!}

\section{Theory for the margin setting}
\label{sec:margin_theory}

We show that gradual self-training does better than directly adapting to the target, where we assume that at each time step there exists some linear classifier---which can be different at each step---that can classify most of the data correctly with a margin (a standard assumption in learning theory), and that the shifts are small. Our main result (Theorem~\ref{thm:gradualSelfTrain}) bounds the error of gradual self-training.
We show that our analysis is tight for gradual self-training (Example~\ref{ex:selfTrainingExponential}), and explain why regularization, label sharpening, and the ramp loss, are key to our bounds.
Proofs are in Appendix~\ref{sec:appendix_margin_theory}.

\subsection{Assumptions}

\textbf{Models and losses:} We consider regularized linear models that have weights with bounded $\ell_2$ norm: $\linearmodels_{\reg} = \{ (w, b) : w \in \mathbb{R}^d, b \in \mathbb{R}, \|w\|_2 \leq \reg \}$ for some fixed $\reg > 0$. Given $(w, b) \in \linearmodels_{\reg}$, the model's output is $M_{w, b}(x) = w^{\top}x + b$.

We consider margin loss functions such as the hinge and ramp losses. Intuitively, a margin loss encourages a model to classify points correctly and confidently---by keeping correctly classified points far from the decision boundary. We consider the hinge function $\hinge$ and ramp function $\ramp$:
\begin{align}
\hinge(m) &= \max(1 - m, 0) \\
\ramp(m) &= \min(\hinge(m), 1)
\end{align}
The ramp loss is $\rampl(\hat{y}, y) = \ramp(y \hat{y})$, where $\hat{y} \in \mathbb{R}$ is a model's prediction, and $y \in \{-1, 1\}$ is the true label.
The hinge loss is the standard way to enforce margin, but the ramp loss is more robust towards outliers because it is bounded above---no single point contributes too much to the loss.
We will see that the ramp loss is key to the theoretical guarantees for gradual self-training because of its robustness.
We denote the population ramp loss as:
\begin{equation} \rampL(\theta, P) = \E_{X, Y \sim P}[\rampl(M_{\theta}(X), Y)] \end{equation}
\tengyu{would it be easier if we just write $\phi(M(X)Y)$}
\ak{Ah I think this will be inconsistent with the self-training algorithm definition though, in the setup}
Given a finite sample $S$, the empirical loss is:
\begin{equation} \rampL(\theta, S) = \frac{1}{|S|} \sum_{x, y \in S}[\rampl(M_{\theta}(x), y)] \end{equation}
% We define $\hingel$ and $\hingeL$ correspondingly for the hinge loss.

\textbf{Distributional distance:} Our notion of distance is $W_{\infty}$, the Wasserstein-infinity distance.
Intuitively, $W_{\infty}$ moves points from distribution $P$ to $Q$ by distance at most $\epsilon$ to match the distributions.
For ease of exposition we consider the Monge form of $W_{\infty}$, although the results can be extended to the Kantarovich formulation as well.
Formally, given probability measures $P, Q$ on $\mathcal{X}$:
% we say $\wasser(P, Q)$ is $\leq \dist$ if there exists some function $f$ mapping $P$ to $Q$, where $f$ does not move any point more than $\dist$ away:
\begin{align}
\wasser(P, Q) = \inf \{ &\sup_{x \in \xspace} ||f(x) - x||_2 : \nonumber\\
&f : \xspace \to \xspace, f_\#P = Q\}
\end{align}

As usual, $\#$ denotes the push-forward of a measure, that is, for every set $A \subseteq \xspace$, $f_\#P(A) = P(f^{-1}(A))$.

In our case, we require that the conditional distributions do not shift too much. Given joint probability measures $P, Q$ on the inputs and labels $\xspace \times \yspace$, the distance is:
\begin{align}
\dist(P, Q) = \max(&\wasser(P_{X \mid Y=1}, Q_{X \mid Y=1}), \nonumber\\
&\wasser(P_{X \mid Y=-1}, Q_{X \mid Y=-1})).
\end{align}

\newcommand{\sepAssump}{$\alpha^*$-separation}
\newcommand{\noLabShiftAssump}{no label shift}
\newcommand{\NoLabShiftAssump}{No label shift}
\newcommand{\gradShiftAssump}{gradual shift}
\newcommand{\GradShiftAssump}{Gradual shift}
\newcommand{\boundedAssump}{bounded data}
\newcommand{\BoundedAssump}{Bounded data}

\textbf{\sepAssump{} assumption}: Assume every domain admits a classifier with low loss $\alpha^*$, that is there exists $\alpha^* \geq 0$ and for every domain $P_t$, there exists some $\theta_t \in \linearmodels_{\reg}$ with $\rampL(\theta_t, P_t) \leq \alpha^*$.

\textbf{\GradShiftAssump{} assumption}: For some $\rho < \frac{1}{\reg}$, assume $\dist(P_t, P_{t+1}) \leq \rho$ for every consecutive domain, where $\frac{1}{\reg}$ is the regularization strength of the model class $\linearmodels_{\reg}$. $\gamma = \frac{1}{\reg}$ can be interpreted as the geometric margin (distance from decision boundary to data) the model is trying to enforce.

\textbf{\BoundedAssump{} assumption}: When dealing with finite samples we need a standard regularity condition: we say that $P$ satisfies the \emph{bounded data assumption} if the data is not too large on average: $\E_{X \sim P}[ ||X||_2^2 ] \leq B^2$ where $B > 0$.

\textbf{\NoLabShiftAssump{} assumption}: Assume that the fraction of $Y=1$ labels does not change: $P_t(Y)$ is the same for all $t$. 

\subsection{Domain shift: baselines fail}

While the distribution shift from $P_t$ to $P_{t+1}$ is small, the distribution shift from the source $P_0$ to the target $P_T$ can be large, as visualized in Figure~\ref{fig:grad_st_intuition}.
A classifier that gets 100\% accuracy on $P_0$, might classify every example wrong on $P_T$, even if $T \geq 2$.
In this case, directly adaptating to $P_T$ would not help.
The following example formalizes this:

\newcommand{\baselinesFailText}{
  Even under the \sepAssump, \noLabShiftAssump, \gradShiftAssump, and \boundedAssump{} assumptions, there exists distributions $P_0, P_1, P_2$ and a source model $\theta \in \linearmodels_{\reg}$ that gets $0$ loss on the source ($\rampL(\theta, P_0) = 0$), but high loss on the target: $\rampL(\theta, P_2) = 1$. Self-training directly on the target does not help: $\rampL(\selftrain(\theta, P_2), P_2) = 1$. This holds true even if every domain is separable, so $\alpha^* = 0$.
}

\begin{example}
\label{ex:baselinesFail}
\baselinesFailText{}
\end{example}

\textbf{Other methods}: Our analysis focuses on self-training, but other bounds do not apply in this setting because they either assume that the density ratio between the target and source exists and is not too small~\cite{huang2006correcting}, or that the source and target are similar enough that we cannot discriminate between them~\cite{ben2010theory}.
% ---see Section~\ref{sec:related_work} for details..

% Two classes of popular methods are: importance-weighting based methods, and invariant representation based methods.
% Bounds on importance-weighting algorithms depend on the density ratio between the source and target, which crucially requires the support of the target to be a subset of the support of the source, which is not necessarily the case.
% Methods that learn invariant representations are inspired by (cite: Ben-David), but due to identifiability issues there are nearly no provable guarantees (cite Han Zhao, Rui).

\subsection{Gradual self-training improves error}

We show that gradual self-training helps over direct adaptation.
For intuition, consider a simple example where $\alpha^* = 0$ and $\theta_0$ classifies every example in $P_0$ correctly with geometric margin $\gamma = \frac{1}{R}$.
If each point shifts by distance $< \gamma$, $\theta_0$ gets every example in the new domain $P_1$ correct.
If we had infinite unlabeled data from $P_1$, we can learn a model $\theta'$ that classifies every example in the new domain $P_1$ correctly with margin $\gamma$ since $\alpha^* = 0$.
Repeating the process for $P_2, \ldots, P_T$, we get every example in $P_T$ correct.

But what happens when we start with a model that has some error, for example because the data cannot be perfectly separated, and have only finite unlabeled samples?
We show that self-training still does better than adapting to the target domain directly, or using the non-adaptive source classifier.

The first main result of the paper says that if we have a model $\theta$ that gets low loss and the distribution shifts slightly, self-training gives us a model $\theta'$ that does not do too badly on the new distribution.

\newcommand{\gradualSelfTrainTheoremText}{
Given $P, Q$ with $\dist(P, Q) = \rho < \frac{1}{\reg}$ and marginals on $Y$ are the same so $P(Y) = Q(Y)$. Suppose $P, Q$ satisfy the \boundedAssump{} assumption, and we have initial model $\theta$, and $n$ unlabeled samples $S$ from $Q$, and we set $\theta' = \selftrain(\theta, S)$. Then with probability at least $1 - \delta$ over the sampling of $S$, letting $\alpha^* = \min_{\theta^* \in \linearmodels_{\reg}} \rampL(\theta^*, Q)$:
\begin{align}
\rampL(\theta', Q) \leq &\frac{2}{1 - \rho \reg} \rampL(\theta, P) + \alpha^* \nonumber\\
&+ \frac{4B \reg + \sqrt{2 \log{2 / \delta}}}{\sqrt{n}}
\end{align}
}

\begin{theorem}
\label{thm:gradualSelfTrain}
\gradualSelfTrainTheoremText{}
\end{theorem}

The proof of this result is in Appendix~\ref{sec:appendix_margin_theory}, but we give a high level sketch here. There exists some classifier that gets accuracy $\alpha^*$ on $Q$, so if we had access to $n$ \emph{labeled} examples from $Q$ then empirical risk minimization gives us a classifier that is accurate on the population---from a Rademacher complexity argument we get a classifier $\theta'$ with loss at most $\alpha^* + O(1/\sqrt{n})$, the second and third term in the RHS of the bound.

Since we only have \emph{unlabeled} examples from $Q$, self-training uses $\theta$ to pseudolabel these $n$ examples and then trains on this generated dataset. Now, if the distribution shift $\rho$ is small relative to the geometric margin $\gamma = \frac{1}{\reg}$, then we can show that the original model $\theta$ labels most examples in the new distribution $Q$ correctly---that is, $\error(\theta, Q)$ is small if $\rampL(\theta, P)$ is small. Finally, if most examples are labeled correctly we show that because there exists some classifier $\theta^*$ with low margin loss, self-training will also learn a classifier $\theta'$ with low margin loss $\rampL(\theta', Q)$, which completes the proof.

We apply this argument inductively to show that after $T$ time steps, the error of gradual self-training is $\lesssim \exp(cT) \alpha_0$ for some constant $c$, if the original error is $\alpha_0$.

\newcommand{\gradualSelfTrainCorollaryText}{
  Under the \sepAssump, \noLabShiftAssump, \gradShiftAssump, and \boundedAssump{} assumptions, if the source model $\theta_0$ has low loss $\alpha_0 \geq \alpha^*$ on $P_0$ (i.e. $\rampL(\theta_0, P_0) \leq \alpha_0$) and $\theta$ is the result of gradual self-training: $\theta = \selftrain(\theta_0, (S_1, \ldots, S_n))$, letting $\beta = \frac{2}{1-\rho \reg}$:
\begin{equation}
\rampL(\theta, P_T) \leq \beta^{T+1} \Big( \alpha_0 + \frac{4B \reg + \sqrt{2 \log{2T / \delta}}}{\sqrt{n}} \Big).
\end{equation}
}

\begin{corollary}
\label{cor:gradualSelfTrain}
\gradualSelfTrainCorollaryText{}
\end{corollary}

Corrollary~\ref{cor:gradualSelfTrain} says that the gradual structure allows some control of the error unlike direct adaptation where the accuracy on the target domain can be 0\% if $T \geq 2$. Note that if the classes are separable and we have infinite data, then gradual self-training maintains 0 error.
% Note that the ramp loss, regularization, and using `hard' labels are key to getting any control of the error (Section~\ref{subsec:essential_ingredients_theory}).
% We show in Section~\ref{subsec:essential_ingredients_theory} that getting any control of the error is tricky, for example self-training with the more standard hinge loss would not enable this.

Our next example shows that our analysis for gradual self-training in this setting is tight---if we start with a model with loss $\alpha_0$, then the error can in fact increase exponentially even with infinite unlabeled examples.
Intuitively, at each step of self-training the loss can increase by a constant factor, which leads to an exponential growth in the error.

\newcommand{\selfTrainingExponentialText}{
  Even under the \sepAssump, \noLabShiftAssump, \gradShiftAssump, and \boundedAssump{} assumptions,
  given $0 < \alpha_0 \leq \frac{1}{4}$, for every $T$ there exists distributions $P_0, \ldots, P_{2T}$, and $\theta_0 \in \linearmodels_{\reg}$ with $\rampL(\theta_0, P_0) \leq \alpha_0$, but if $\theta' = \selftrain(\theta_0, (P_1, \ldots, P_{2T}))$ then $\rampL(\theta', P_{2T}) \geq \min(0.5, \frac{1}{2} 2^T \alpha_0)$. Note that $\rampL$ is always in $[0, 1]$.
}

\begin{example}
\label{ex:selfTrainingExponential}
\selfTrainingExponentialText{}
\end{example}

This suggests that if we want sub-exponential bounds we either need to make additional assumptions on the data distributions, or devise alternative algorithms to achieve better bounds (which we believe is unlikely).

\subsection{Essential ingredients for gradual self-training}
\label{subsec:essential_ingredients_theory}

In this section, we explain why regularization, label sharpening, and the ramp loss are essential to bounding the error of gradual self-training (Theorem~\ref{thm:gradualSelfTrain}).

\textbf{Regularization}:
Without regularization there is no incentive for the model to change when self-training---if we self-train without regularization an optimal thing to do is to output the original model. The intuition is that since the model $\theta = (w, b)$ is used to pseudolabel examples, $\theta$ gets every pseudolabeled example correct. The scaled classifier $\theta' = (\alpha w, \alpha b)$ for large $\alpha$ then gets optimal loss, but $\theta'$ and $\theta$ make the same predictions for every example. We use $\selftrain'(\theta, S)$ to denote the \emph{set} of possible $\theta'$ that minimize the loss on the pseudolabeled distribution (Equation~\eqref{eqn:selfTrainSample}):

\newcommand{\noRegularizationNoGainText}{
  Given a model\pl{if you use $\Theta$ notation, then you can just say $R = \infty$, which is easier to parse anyway} $\theta \in \linearmodels_{\infty}$ and unlabeled examples $S$ where for all $x \in S$, $M_{\theta}(x) \neq 0$, there exists $\theta' \in \selftrain'(\theta, S)$ such that for all $x \in \xspace$, $M_{\theta}(x) = M_{\theta'}(x)$.
}

\begin{example}
\label{ex:noRegularizationNoGain}
\noRegularizationNoGainText{}
\end{example}
\pl{should this really be called an example? I feel like these are more results (propositions?) rather than simple illustrative examples}
\ak{I think they are fairly straightforward? It seems too simple to be a proposition.}

More specific to our setting, our bounds require regularized models because regularized models classify the data correctly with \emph{a margin}, so even after a mild distribution shift we get most new examples correct.
Note that in traditional supervised learning, regularization is usually required when we have few examples for better generalization to the population, whereas in our setting regularization is important for maintaining a margin even with infinite data.

\textbf{Label sharpening}: When self-training, we pseudolabel examples as $-1$ or $1$, based on the output of the classifier.
Prior work sometimes uses ``soft" labels~\cite{najafi2019robustness}, where for each example they assign a probability of the label being $-1$ or $1$, and train using a logistic loss.
% More precisely, let $\sigma$ be the sigmoid function---given an arbitrary distribution $P$ and model $M_{\theta}$, $\theta \in \Theta$, 
The loss on the soft-pseudolabeled distribution is defined as:
\newcommand{\logL}{\ensuremath{L_{\sigma, \theta}}}
\newcommand{\logl}{\ensuremath{ll}}
\begin{equation} \logL(\theta') = \E_{X \sim P} [\logl(\sigma(M_{\theta}(X)), \sigma(M_{\theta'}(X)))] \end{equation},
where $\sigma$ is the sigmoid function, and $\logl$ is the log loss:
\begin{equation} \logl(p, p') = p \log p' + (1-p) \log{(1-p')} \end{equation}
Self-training then picks $\theta' \in \Theta$ minimizing $\logL(\theta')$.
A simple example shows that this form of self-training may never update the parameters because $\theta$ minimizes $\logL$:

\newcommand{\softLabelsBadText}{
For all $\theta \in \Theta$, $\theta$ is a minimizer of $\logL{}$, that is, for all $\theta' \in \Theta$, $\logL(\theta) \leq \logL(\theta')$.

}

\begin{example}
\label{ex:softLabelsBad}
\softLabelsBadText{}
\end{example}

This suggests that we ``sharpen" the soft labels to encourage the model to update its parameters.
Note that this is true even on finite data: set $P$ to be the empirical distribution.

\textbf{Ramp versus hinge loss}:
We use the ramp loss, but does the more popular hinge loss $\hingeL$ work?
Unfortunately, the next example shows that we cannot control the error of gradual self-training with the hinge loss even if we had infinite examples, so the ramp loss is important for Theorem~\ref{thm:gradualSelfTrain}.

\newcommand{\HingeLossBadText}{
  Even under the \sepAssump, \noLabShiftAssump, and \gradShiftAssump{} assumptions,
  given $\alpha_0 > 0$, there exists distributions $P_0, P_1, P_2$ and $\theta_0 \in \linearmodels_{\reg}$ with $\hingeL(\theta_0, P_0) \leq \alpha$, but if $\theta' = \selftrain(\theta_0, (P_1, P_2))$ then $\hingeL(\theta', P_2) \geq \error(\theta', P_2) = 1$ ($\theta'$ gets every example in $P_2$ wrong), where we use the hinge loss in self-training.
}

\begin{example}
\label{ex:hingeLossBad}
\HingeLossBadText{}
\end{example}

We only analyzed the statistical effects here---the hinge loss tends to work better in practice because it is much easier to \emph{optimize} and is convex for linear models.

\subsection{Self-training without domain shift}

Example~\ref{ex:selfTrainingExponential} showed that when the distribution shifts, the loss of gradual self-training can grow exponentially (though the non-adaptive baseline has unbounded error).
Here we show that if we have no distribution shift, the error can only grow linearly: if $P_0 = \ldots = P_T$, given a classifier with loss $\alpha_0$, if we do gradual self-training the loss is at most $\alpha_0 T$.

\newcommand{\selfTrainingNoShiftBoundText}{
  Given $\alpha_0 > 0$, distributions $P_0 = … = P_T$, and model $\theta_0 \in \linearmodels_{\reg}$ with $\rampL(\theta_0, P_0) \leq \alpha_0$, $\rampL(\theta', P_T) \leq \alpha_0 (T+1)$ where $\theta' = \selftrain(\theta_0, (P_1, \ldots, P_T))$
}

\begin{proposition}
\label{prop:selfTrainingNoShiftBound}
\selfTrainingNoShiftBoundText{}
\end{proposition}

In Appendix~\ref{sec:appendix_margin_theory}, we show that self-training can indeed hurt without domain shift: given a classifier with loss $\alpha$ on $P$, self-training on $P$ can increase the classifier's loss on $P$ to $2\alpha$, but here the non-adaptive baseline has error $\alpha$.

% This suggests adding more structure to the problem, to investigate when self-training can have better control of the error.
% \newcommand{\selfTrainingHurtsText}{
%   Given $\alpha > 0, \epsilon > 0$, there exists distribution $P$ and model $\theta \in \linearmodels_{\reg}$ with $\rampL(\theta, P_0) \leq \alpha$, but $\rampL(\theta', P) \geq (2 - \epsilon) \alpha$ where $\theta' = \selftrain(\theta, P)$, for any $\epsilon > 0$.
% }

% \begin{example}
% \label{ex:selfTrainingHurts}
% \selfTrainingHurtsText{}
% \end{example}

\pl{nice results, but it's a bit too bad that you don't give any details of
the actual examples; space is a concern but are they complex? is there anything you can do?}

\section{Theory for the Gaussian setting}
\label{sec:gaussian_theory}

% In Section~\ref{sec:margin_theory} we showed that gradual self-training does better than other alternatives in a fairly natural distribution-free setting.
% However, when the classes are not linearly separable, the loss of self-training can grow exponentially even if we have infinite examples (Example~\ref{ex:selfTrainingExponential}) which is impractical.
% Can self-training do better than the exponential bound if the data behaves `nicely'?

In this section we study an idealized Gaussian setting to understand conditions under which self-training can have better than exponential error bounds: we show that if we begin with a good classifier, the distribution shifts are not too large, and we have infinite \emph{unlabeled} data, then gradual self-training maintains a good classifier.

\subsection{Setting}

We assume $P_t(X \mid Y=y)$ is an isotropic Gaussian in $d$-dimensions for each $y \in \{-1, 1\}$.
We can shift the data to have mean $0$, so we suppose:
\begin{equation} P_t(X | Y=y) = \mathcal{N}(y \mu_t, \sigma_t^2 I) \end{equation}
Where $\mu_t \in \mathbb{R}^d$ and $\sigma_t > 0$ for each $t$.
\pl{uh, you can do that for one time step, but doesn't make sense across time steps? doesn't seem fully general}
\ak{If you have infinite unlabeled data, and it’s isotropic, you can just shift the data in the current domain its mean. If the original shifts between domain 1 and 2 is small, then the shifts in the mean-centered data from 1 and 2 will be small as well.}
As usual, we assume the shifts are gradual: for some $B > 0$, $\|\mu_{t+1} - \mu_t\|_2 \leq \frac{B}{4}$.
We assume that the means of the two classes do not get closer than the shift, or else it would be impossible to distinguish between no shift, and the distributions of the two classes swapping: so $\|\mu_t\|_2 \geq B$ for all $t$.
We assume infinite unlabeled data (access to $P_t(X)$) in our analysis.

Given labeled data in the source, we use the objective:
\begin{equation} L(w, P) = \E_{X, Y \sim P}[\phi(Y(w^{\top}X))] \end{equation}
For unlabeled data, self-training performs descent steps on an underlying objective function~\cite{amini2003semisupervised}, which we focus on:
\begin{equation} U(w, P) = \E_{X \sim P}[\phi(\lvert w^{\top}X \rvert)] \end{equation}
We assume $\phi : \mathbb{R} \to \mathbb{R}^+$ is a continuous, non-increasing function which is strictly decreasing on $[0, 1]$: these are regularity conditions which the hinge, ramp, and logistic losses satisfy. If $w' = \selftrain(w, P)$ then $U(w', P) \leq U(w, P)$~\cite{amini2003semisupervised}.
% The unlabeled objective behaves similarly to entropy minimization (Cite: semi-supervised book, pseudolabel).
% ---this does not guarantee converges even to a local minimum but we set aside these optimization issues.
\pl{why not just use $\theta$ instead of $w$ to make the notation consistent with previous section?}
\ak{there's no bias term here, and I use $w$ for the weights before (and in the appendices), is that fine?}

The algorithm we analyze begins by choosing $w_0$ from labeled data in $P_0$, and then updates the parameters with unlabeled data from $P_t$ for $1 \leq t \leq T$:
\begin{align}
w_t &= \argmin_{\|w\|_2 \leq 1, \|w - w_{t-1}\|_2 \leq \frac{1}{2}} U(w, P_t) \label{eqn:constrained_min}
\end{align}

Note that we do not show that self-training actually converges to the constrained minimum of $U$ in Equation~\eqref{eqn:constrained_min} and prior work only shows that self-training descends on $U$---we leave this optimization analysis to future work.

\subsection{Analysis}

Let $w^*(\mu) = \frac{\mu}{\|\mu\|_2}$ where $\|\mu\| \geq B > 0$.
Note that $w^*(\mu_t)$ minimizes the 0-1 error on $P_t$.
Our main theorem says that if we start with a regularized classifier $w_0$ that is near $w^*(\mu_0)$, which we can learn from labeled data, and the distribution shifts $\|\mu_{t+1} - \mu_t\|_2$ are not too large, then we recover the optimal $w_T = w^*(\mu_T)$.
The key challenge is that the unlabeled loss $U$ in $d$ dimensions is non-convex, with multiple local minima, so directly minimizing $U$ does not guarantee a solution that minimizes the labeled loss $L$.

\newcommand{\gaussianTheoremText}{
Assuming the Gaussian setting, if $\|w_0 - w^*(\mu_0)\|_2 \leq \frac{1}{4}$, then we recover $w_T = w^*(\mu_T)$.
}

\begin{theorem}
\label{thm:gaussian}
\gaussianTheoremText{}
\end{theorem}

Proving this reduces to proving the \emph{single-step} case. At each step $t+1$, if we have a classifier $w_t$ that was close to $w^*(\mu_t)$, then we will recover $w_{t+1} = w^*(\mu_{t+1})$. We give intuition here and the formal proof in Appendix~\ref{sec:appendix_gaussian_theory}.

We first show that if $\mu$ changes by a small amount, the optimal parameters (for the labeled loss) does not change too much.
Then since $w_t$ is close to $w^*(\mu_t)$, $w_t$ is not too far away from $w^*(\mu_{t+1})$.
The key step in our argument is showing that the unique minimum of the unlabeled loss $U(w, P_{\mu_{t+1}})$ in the neighborhood of $w_t$, is $w^*(\mu_t)$---looking for a minimum \emph{nearby} is important because if we deviate too far we might select other ``bad" minima.
We consider arbitrary $w$ near $w^*(\mu_{t+1})$ and construct a pairing of points $(a,b)$ in $\mathbb{R}^d$, using a convexity argument to show that $(a, b)$ contributes more to the loss of $w$ than $w^*(\mu_{t+1})$.
% Our argument works for a slightly more general class of distributions: if the density $f$ at $x$ can be written as $f(x) = kg(\|x\|_2)$, where $k$ is a normalizing constant, and $g$ is strictly convex and positive.

% This theorem assumes infinite unlabeled samples, and to extend to finite samples we need to concentrate the empirical loss near the population loss---we hope future work can extend to finite samples and beyond the Gaussian setting.

\section{Experiments}
\label{sec:experiments}

Our theory leads to practical insights---we show that regularization and label sharpening are important for gradual self-training, that leveraging the gradual shift structure improves target accuracy, and give intuition for when the gradual shift assumption may not help.
We run experiments on three datasets (see Appendix~\ref{sec:appendix_experiments} for more details):
% \tnote{I think it makes sense to have a bit longer introduction. List the main questions you are investigating in each subsections. }

\textbf{Gaussian}: Synthetic dataset where the distribution $P_t(X | Y)$ for each of two classes is a $d$-dimensional Gaussian, where $d=100$. The means and covariances of each class vary over time. The model gets $500$ labeled samples from the source domain, and $500$ unlabeled samples from each of $10$ intermediate domains. This dataset resembles our Gaussian setting but the covariance matrices are not isotropic, and the number of labeled and unlabeled samples is finite and on the order of the dimension $d$.

\textbf{Rotating MNIST}: Rotating MNIST is a semi-synthetic dataset where we rotate each MNIST image by an angle between 0 and 60 degrees. We split the 50,000 MNIST training set images into a source domain (images rotated between 0 and 5 degrees), intermediate domain (rotations between 5 and 60 degrees), and a target domain (rotations between 55 degrees and 60 degrees). Note that each image is seen at exactly one angle, so the training procedure cannot track a single image across different angles.
% \pl{cite if anyone does this}

\textbf{Portraits}: A real dataset comprising photos of high school seniors across years~\cite{ginosar2017portraits}. The model's goal is to classify gender. We split the data into a source domain (first 2000 images), intermediate domain (next 14000 images), and target domain (next 2000 images).

\subsection{Does the gradual shift assumption help?}
\label{sec:doesGradualShiftHelpExperiments}

Our goal is to see if adapting to the gradual shift sequentially helps compared to directly adapting to the target. We evaluate four methods: \emph{Source}: simply train a classifier on the labeled source examples. \emph{Target self-train}: repeatedly self-train on the unlabeled target examples ignoring the intermediate examples. \emph{All self-train}: pool all the unlabeled examples from the intermediate and target domains, and repeatedly self-train on this pooled dataset to adapt the initial source classifier. \emph{Gradual self-train}: sequentially use self-training on unlabeled data in each successive intermediate domain, and finally self-train on unlabeled data on the target domain, to adapt the initial source classifier.

% \begin{enumerate}
% \item (Source) The naive baseline is to simply train a classifier on the labeled source examples.
% \item (Target self-train) Only self-train on the unlabeled target examples ignoring the intermediate examples.
% \item (All self-train) Pool all the unlabeled examples from the intermediate and target domains, and self-train on this pooled dataset to adapt the initial source classifier.
% \item (Gradual self-train) Sequentially use self-training on unlabeled data in each successive intermediate domain, and finally self-train on unlabeled data on the target domain, to adapt the initial source classifier.
% \end{enumerate}
For the Gaussian and MNIST datasets, we ensured that the target self-train method sees as many unlabeled target examples as gradual self-train sees across all the intermediate examples.
% That is, if gradual self-train consumed $N=\sum_{i=1}^T n_i$ unlabeled examples in total (from intermediate domains), we gave target self-train $N$ unlabeled examples from the target domain.
Since portraits is a real dataset we cannot synthesize more examples from the target, so target self-train uses fewer unlabeled examples here. 

For rotating MNIST and Portraits we used a 3-layer convolutional network with dropout$(0.5)$ and batchnorm on the last layer, that was able to achieve $97\%-98\%$ accuracy on held out examples in the source domain.
For the Gaussian dataset we used a logistic regression classifier with $l_2$ regularization.
For each step of self-training, we filter out the 10\% of images where the model's prediction was least confident---Appendix~\ref{sec:appendix_experiments} shows similar findings without this filtering\pl{then why bother with this filtering since it's more complicated?}\ak{all methods do better with this filtering, but the relative difference is similar, see appendix C}.
To account for variance in initialization and optimization, we ran each method 5 times and give $90\%$ confidence intervals.
More experimental details are in Appendix~\ref{sec:appendix_experiments}.

\begin{table}[t]
\caption{
Classification accuracies for gradual self-train (ST) and baselines on 3 datasets, with $90\%$ confidence intervals for the mean over 5 runs.
Gradual ST does better than self-training directly on the target or self-training on all the unlabeled data pooled together.
}
\label{tab:gradHelps}
\vskip 0.15in
\begin{center}
\begin{small}
\begin{sc}
\begin{tabular}{lcccr}
\toprule
 & Gaussian & Rot MNIST & Portraits \\
\midrule
Source              & 47.7$\pm$0.3  & 31.9$\pm$1.7 & 75.3$\pm$1.6 \\
Target ST   & 49.6$\pm$0.0  & 33.0$\pm$2.2 & 76.9$\pm$2.1 \\
All ST     & 92.5$\pm$0.1  & 38.0$\pm$1.6 & 78.9$\pm$3.0 \\
Gradual ST & \textbf{98.8$\pm$0.0} & \textbf{87.9$\pm$1.2} & \textbf{83.8$\pm$0.8} \\
\bottomrule
\end{tabular}
\end{sc}
\end{small}
\end{center}
\vskip -0.1in
\end{table}

% The accuracy of gradual self-training drops for each intermediate domain shift, but the final performance on the target is still better.
Table~\ref{tab:gradHelps} shows that leveraging the gradual structure leads to improvements over the baselines on all three datasets.

\subsection{Important ingredients for gradual self-training}

Our theory suggests that regularization and label sharpening are important for gradual self-training, because without regularization and label sharpening there is no incentive for the model to change (Section~\ref{subsec:essential_ingredients_theory}).
However, prior work suggests that overparameterized neural networks trained with stochastic gradient methods have strong implicit regularization~\cite{zhang2017understanding, hardt2016train}---in the supervised setting they perform well without explicit regularization even though the number of parameters is much larger than the number of data points---is this implicit regularization enough for gradual self-training?

In our experiments, we see that even without explicit regularization, or with `soft' probabilistic labels, gradual self-training does slightly better than the non-adaptive source classifier, suggesting that this implicit regularization may have some effect.
However, explicit regularization and `hard' labeling gives a much larger accuracy boost.

\textbf{Regularization is important}: We repeat the same experiment as Section~\ref{sec:doesGradualShiftHelpExperiments}, comparing gradual self-training with or without regularization---that is, disabling dropout and batchnorm~\cite{ioffe2015batch} in the neural network experiments.
In both cases, we first train an \emph{unregularized} model on labeled examples in the source domain.
Then, we either turn on regularization during self-training, or keep the model unregularized.
We control the original model to be the same in both cases to see if regularization helps in the self-training process, as opposed to in learning a better supervised classifier.
Table~\ref{tab:regHardLabelImportant} shows that accuracies are significantly better with regularization, even though unregularized performance is still better than the non-adaptive source classifier. \tnote{cite some paper that proposes that batchnorm may have regularization effect? Fixup discuss it a bit but very little, which other paper shows it?} \ak{Cited the original batch norm paper which claims this}

\textbf{Soft labeling hurts}: We ran the same experiment as Section~\ref{sec:doesGradualShiftHelpExperiments}, comparing gradual self-training with hard labeling versus using probabilistic labels output by the model. Table~\ref{tab:regHardLabelImportant} shows that accuracies are better with hard labels.

\begin{table}[t]
\caption{
Classification accuracies for gradual self-train with explicit regularization and hard labels (Gradual ST), without regularization but with hard labels (No Reg), and with regularization but with soft labels (Soft Labels).
Gradual self-train does best with explicit regularization and hard labels, as our theory suggests, even for neural networks with implicit regularization.
}
\label{tab:regHardLabelImportant}
\vskip 0.15in
\begin{center}
\begin{small}
\begin{sc}
\begin{tabular}{lcccr}
\toprule
 & Gaussian & Rot MNIST & Portraits \\
\midrule
Soft Labels         & 90.5$\pm$1.9  & 44.1$\pm$2.3 & 80.1$\pm$1.8 \\
No Reg   & 84.6$\pm$1.1  & 45.8$\pm$2.5 & 76.5$\pm$1.0 \\
Gradual ST           & \textbf{99.3$\pm$0.0}  & \textbf{83.8$\pm$2.5} & \textbf{82.6$\pm$0.8} \\
\bottomrule
\end{tabular}
\end{sc}
\end{small}
\end{center}
\vskip -0.1in
\end{table}

\textbf{Regularization is still important with more data}:
In supervised learning, the importance of regularization diminishes as we have more training examples---if we had access to infinite data (the population), we don't need regularization.
On the other hand, for gradual domain adaptation, the theory says regularization is needed to adapt to the dataset shift even with infinite data, and predicts that regularization remains important even if we increase the sample size.

To test this hypothesis, we construct a rotating MNIST dataset where we increase the sample sizes.
The source domain $P_0$ consists of $N \in \{2000, 5000, 20000\}$ images on MNIST.
$P_t$ then consists of these \emph{same} $N$ images, rotated by angle $3t$, for $0 \leq t \leq 20$.
% We have access to labeled source examples and unlabeled intermediate examples: $P_0(X, Y)$, and $P_i(X)$ for $1 \leq i \leq 20$. 
The goal is to get high accuracy on $P_{20}$: these images rotated by 60 degrees---the model doesn't have to generalize to unseen images, but to seen images at different angles.
We compare using regularization versus not using regularization during gradual self-training.
\begin{table}[t]
\caption{
Classification accuracies for gradual self-train on rotating MNIST as we vary the number of samples.
Unlike in previous experiments, here the same $N$ samples are rotated, so the models do not have to generalize to unseen images, but seen images at different angles.
The gap between regularized and unregularized gradual self-training does not shrink much with more data.
}
\label{tab:moreDataDoesntHelp}
\vskip 0.15in
\begin{center}
\begin{small}
\begin{sc}
\begin{tabular}{lcccr}
\toprule
 & N=2000 & N=5000 & N=20,000 \\
\midrule
Source         & 28.3$\pm$1.4  & 29.9$\pm$2.5 & 33.9$\pm$2.6\\
No Reg   & 55.7$\pm$3.9  & 53.6$\pm$4.0 & 55.1$\pm$3.9\\
Reg           & \textbf{93.1$\pm$0.8}  & \textbf{91.7$\pm$2.4} & \textbf{87.4$\pm$3.1}\\
\bottomrule
\end{tabular}
\end{sc}
\end{small}
\end{center}
\vskip -0.1in
\end{table}
Table~\ref{tab:moreDataDoesntHelp} shows that regularization is still important here, and the gap between regularized and unregularized gradual self-training does not shrink much with more data.

\pl{one could just say maybe your $N$ is too small, and if you had a million points, then it wouldn't matter...}
\ak{True, I added an experiment with 20K points.}

\subsection{When does gradual shift help? \tnote{can we have a statement instead of a question?}}

Our theory in Section~\ref{sec:margin_theory} says that gradual self-training works well if the shift between domains is small in Wasserstein-infinity distance, but it may not be enough for the total variation or KL-divergence between $P$ and $Q$ to be small.
% Intuitively, the Wasserstein-infinity distance between $P$ and $Q$ is small if there exists some pairing of points across $P$ and $Q$ so that each point doesn't move more than distance $\epsilon$ when going from $P$ to $Q$.

To test this, we run an experiment on a modified version of the rotating MNIST dataset.
We keep the source and target domains the same as before, but change the intermediate domains.
In Table~\ref{tab:gradHelps} we saw that gradual self-training works well if we have intermediate images rotated by gradually increasing rotation angles.
Another type of gradual transformation is to gradually introduce more examples rotated by $55$ to $60$ degrees.
That is, in the $i$-th domain, $(20-i)/20$ fraction of the examples are MNIST images rotated by $0$ to $5$ degrees, and $i/20$ of the examples are MNIST images rotated by $55$ to $60$ degrees, where $1 \leq i \leq 20$.
Here the total-variation distance between successive domains is small, but intuitively the Wasserstein distance is large because each image undergoes a large ($\approx 55$ degrees) rotation.

As the theory suggests, here gradual self-training does not outperform directly self-training on the target---gradual self-training gets $33.5 \pm 1.5\%$ accuracy on the target, while direct adaptation to the target gets $33.0 \pm 2.2\%$ over 5 runs.
We hope this gives practitioners some insight into not just the strengths of gradual self-training, but also its limitations.

\section{Related work}
\label{sec:related_work}

\textbf{Self-training} is a popular method in semi-supervised learning~\cite{lee2013pseudo, sohn2020fixmatch} and domain adaptation~\cite{long2013transfer, zou2019confidence, inoue2018cross}, and is related to entropy minimization~\cite{grandvalet05entropy}. Theory in semi-supervised learning~\cite{rigollet2007generalization, singh2008unlabeled, shai2008unlabeled} analyzes when unlabeled data can help, but does not show bounds for particular algorithms. Recent work shows that a robust variant of self-training can mitigate the tradeoff between standard and adversarial accuracy~\cite{raghunathan2020understanding}. Related to self-training is co-training~\cite{blum98cotraining}, which assumes that the input features can be split into two or more views that are conditionally independent on the label.

Unsupervised \textbf{domain adaptation}, where the goal is to directly adapt from a labeled source domain to an unlabeled target domain, is widely studied~\cite{quinonero2009dataset}.
The key challenge for domain adaptation theory is when the source and target supports do not overlap~\cite{zhao2019zhao, shu2018dirtt}, which are typical in the modern high-dimensional regime.
\emph{Importance weighting} based methods~\cite{shimodaira2000improving, sugiyama2007covariate, huang2006correcting} assume the domains overlap, with bounds depending on the expected density ratios between the source and target.
Even if the domains overlap, the density ratio often scales exponentially in the dimension.
These methods also assume that $P(Y \mid X)$ is the same for the source and target.
\emph{The theory of $H\Delta H$-divergence}~\cite{ben2010theory, mansour2009domain} gives conditions for when a model trained on the source does well on the target \emph{without any adaptation}.
Empirical methods aim to learn domain invariant representations~\cite{tzeng2014domain, ganin2015domain, tzeng2017domain} but there are no theoretical guarantees for these methods~\cite{zhao2019zhao}.
These methods require additional heuristics~\cite{hoffman2018cycada}, and work well on some tasks but not others~\cite{bobu2018adapting, peng2019moment}.
Our work suggests that the structure from gradual shifts, which appears often in applications, can be a way to build theory and algorithms for regimes where the source and target are very different.

% The theory of $H\Delta H$-divergence~\cite{ben2010theory, mansour2009domain} shows that if we cannot discriminate between the source and target domains, and there exists a single classifier that performs well on both domains, then a classifier trained on the source does well on the target \emph{without any adaptation}.
% These conditions are rarely satisfied in the input space, so empirical methods aim to learn an \emph{invariant representation}, one where we cannot tell if it came from the source or target~\cite{tzeng2014domain, ganin2015domain, tzeng2017domain}.
% However, in this learned representation space, there may not exist a classifier that gets high accuracy on both the source and target even if such a classifier exists in the input space~\cite{zhao2019zhao, shu2018dirtt}.
% Formally, these methods optimize some terms in the $H \Delta H$ divergence bounds, but the $\lambda$ term in Theorem 2~\cite{ben2010theory} can grow arbitrarily, and there is no theory for when these methods work.
% Empirically, these methods require additional heuristics~\cite{hoffman2018cycada}, and work on some tasks but not others~\cite{bobu2018adapting, peng2019moment}.

\citet{hoffman2014continuous, gadermayr2018gradual, wulfmeier2018incremental, bobu2018adapting} among others propose approaches for \textbf{gradual domain adaptation}.
% and show experimentally that this can outperform direction adaptation to the target.
% \citet{bobu2018adapting} tackle the problem of catastrophic forgetting in gradual domain adaptation: their goal is to get a model $M$ that performs well on every domain in a sequence of gradually shifting domains.
This setting differs from online learning~\cite{shalev07online}, lifelong learning~\cite{silver2013lifelong}, and concept drift~\cite{kramer1988learning, bartlett1992learning, bartlett1996learning}, since we only have unlabeled data from shifted distributions. To the best of our knowledge, we are the first to develop a theory for gradual domain adaptation, and investigate when and why the gradual structure helps.

\paragraph{Acknowledgements.}

The authors would like to thank the Open Philantropy Project and the Stanford Graduate Fellowship program for funding. This work is also partially supported by the Stanford Data Science Initiative and the Stanford Artificial Intelligence Laboratory.

We are grateful to Stephen Mussman, Robin Jia, Csaba Szepesvari, Shai Ben-David, Lin Yang, Rui Shu, Michael Xie, Aditi Raghunathan, Yining Chen, Colin Wei, Pang Wei Koh, Fereshte Khani, Shengjia Zhao, and Albert Gu for insightful discussions.

\paragraph{Reproducibility.}

Our code is at \url{https://github.com/p-lambda/gradual_domain_adaptation}. Code, data, and experiments will be available on CodaLab soon.

\bibliographystyle{unsrtnat}
\bibliography{local,refdb/all}

\appendix

\newpage
\section{Proofs for Section~\ref{sec:margin_theory}}
\label{sec:appendix_margin_theory}

\newtheorem*{baselinesFailExample}{Restatement of Example~\ref{ex:baselinesFail}}

\begin{baselinesFailExample}
\baselinesFailText{}
\end{baselinesFailExample}

\begin{proof}
We construct an example in 2-D, where we consider the set of regularized linear models $\linearmodels_{\reg}$, where $\reg = 1$. Such a classifier is parametrized by $(w, b)$ where $w \in \mathbb{R}^2$ with $||w||_2 \leq 1$, and $b \in \mathbb{R}$. The output of the model is $M_{w, b}(x) = w^Tx+b$, and the predicted label is $\sign(w^Tx + b)$.

We first define the source distribution $P_0$:
\begin{equation} P_0(X = (1, 1) \land Y = 1) = 0.5 \end{equation}
\begin{equation} P_0(X = (-1, -1) \land Y = -1) = 0.5 \end{equation}
Consider the source classifier $w_0 = (0, 1)$.
The classifier classifies all examples correctly, in particular $\sign(w_0^T (1, 1)) = 1$, and $\sign(w_0^T (-1, -1)) = -1$.
In addition, the ramp loss is $0$, that is:
\begin{equation} \E_{X, Y \sim P_0}[\ramp(Y(w_0^T X))] = 0 \end{equation}
We now construct distributions $P_1$ and $P_2$:
\begin{equation} P_1(X = (1, 1/3) \land Y = 1) = 0.5 \end{equation}
\begin{equation} P_1(X = (-1, -1/3) \land Y = -1) = 0.5 \end{equation}
\begin{equation} P_2(X = (1, -1/3) \land Y = 1) = 0.5 \end{equation}
\begin{equation} P_2(X = (-1, 1/3) \land Y = -1) = 0.5 \end{equation}
Basically, the second-coordinate starts at 1 and decreases over time when the label is $Y = 1$, and starts at $-1$ and increases over time when the label is $Y=-1$.
We note that $\dist(P_0, P_1) = \dist(P_1, P_2) = \frac{2}{3} \leq \frac{1}{\reg}$.

Now, $w_0, b_0$ classifies everything incorrectly in $P_2$.
$\sign(w_0^T (1, -1/3)) = -1$, and $\sign(w_0^T (-1, 1/3)) = 1$ but the corresponding labels in $P_2$ are $1$ and $-1$ respectively.
Accordingly, the ramp loss $\rampL(M_{w_0, b_0}, P_2) = 1$.

Self-traning on $P_2$ cannot fix the problem.
$w_0, b_0$ gets every example incorrect, so all the pseudolabels are incorrect.
In particular, let $Y'$ be the pseudolabels produced using $w_0, b_0$---we have, $Y' \mid [X = (1, -1/3)] = -1$ and $Y' \mid [X = (-1, 1/3)] = 1$.
Self-training on this is now a convex optimization problem, which attains 0 loss, for example using the classifier $w' = (-1, 0)$, $b' = 0$, but any such classifier also gets all the examples incorrect.
Note that the max-margin classifier on the source also exhibits the same issue (that is, it can get all the examples wrong after the dataset shift), from a simple extension of this example.

Finally, the classifier $w^* = (1, 0)$, $b^* = 0$, gets every label correct in \emph{all distributions}, $P_0, P_1, P_2$.

\end{proof}

\newtheorem*{gradualSelfTrainTheorem}{Restatement of Theorem~\ref{thm:gradualSelfTrain}}

\begin{gradualSelfTrainTheorem}
\gradualSelfTrainTheoremText{}
\end{gradualSelfTrainTheorem}

We begin by stating and proving some lemmas that formalize the proof outline in the main paper.
We begin with a standard lemma that says if we learn a regularized linear classifier from $n$ labeled examples from a distribution $P$, then the classifier is almost as good as the optimal regularized linear classifier on $P$, and the classifier gets closer to optimal as $n$ increases. We bound the error of the classifier using the Rademacher complexity of regularized linear models $\linearmodels_{\reg}$.

\begin{lemma}
\label{lem:finiteSampleBound}
Given $n$ samples $S$ from a joint distribution $P$ over inputs $\xspace$ and labels $\yspace$, and suppose $\E_{X \sim P}[||X||_2^2] \leq B^2$. Let $\hat{f}$ and $f$ be the empirical and population minimizers of the ramp loss respectively:
\begin{equation} \hat{f} = \argmin_{f \in \linearmodels_{\reg}} \rampL(f, S) \end{equation}
\begin{equation} f^* = \argmin_{f \in \linearmodels_{\reg}} \rampL(f, P) \end{equation}
Then with probability at least $1 - \delta$,
\begin{equation}
\rampL(\hat{f}) - \rampL(f^*) \leq \frac{4B \reg + \sqrt{2 \log{2 / \delta}}}{\sqrt{n}}
\end{equation}
\end{lemma}

\begin{proof}
We begin with a standard bound (see e.g. Theorem 9, page 70 in~\cite{liang2016statistical}), where the generalization error on the left is bounded by the Rademacher complexity:
\begin{equation} \rampL(\hat{f}) - \rampL(f^*) \leq 4 R_n(A) + \sqrt{\frac{2 \log{2 / \delta}}{n}} \end{equation}
Here, $A = \{ (x, y) \mapsto \rampl(M_{\theta}(x), y) : \theta \in \Theta \}$ is the composition of the loss with the set of regularized linear models, and $R_n$ is the Rademacher complexity.
It now suffices to bound $R_n$.

We first use Talagrand's lemma, which says that if $\phi : \mathbb{R} \to \mathbb{R}^+$ is an $L$-Lipschitz function (that is, $|\phi(b) - \phi(a)| \leq L|b - a|$ for all $a, b$), then:
\begin{equation} R_n(\phi \circ F) \leq L R_n(F) \end{equation}
In our case, we let $F = \{ (x, y) \mapsto y M_{\theta}(x) : \theta \in \Theta \}$, in which case $A = \ramp \circ F$ where $\ramp$ is the ramp loss.
The Lipschitz constant of the ramp loss $\ramp$ is 1, so $R_n(A) \leq R_n(F)$.

Finally, we need to bound $R_n(F)$, the Rademacher complexity of $\ell_2$-regularized linear models. This is a standard argument (e.g. see Theorem 11, page 82 in~\cite{liang2016statistical}) and we get:
\begin{equation} R_n(F) \leq \frac{B \reg}{\sqrt{n}} \end{equation}
\end{proof}

The next lemma shows that the error (0-1 loss) of $M_{\theta}$ is low on $Q$, even though the margin loss may be high. Intuitively, $M_{\theta}$ classifies most points in $P$ correctly with geoemtric margin $\frac{1}{\reg}$, so after a small distribution shift $< \frac{1}{\reg}$, these points are still correctly classified since the margin acts as a `buffer' protecting us from misclassification.

\begin{lemma}
\label{lem:boundErrorFromMargin}
If $\theta \in \linearmodels_{\reg}$, $\rho(P, Q) = \rho < \frac{1}{\reg}$, and the marginals on $Y$ are the same so $P(Y) = Q(Y)$, then $\error(M_{\theta}, Q) \leq \frac{2}{1 - \rho \reg} \rampL(M_{\theta}, P)$ 
\end{lemma}

\begin{proof}
Let $\theta = (w, b)$ be the weights and bias of the regularized linear model, with $||w||_2 \leq R$.

Intuitively, if the ramp loss for a regularized linear model is low, then most points are classified correctly with high geometric margin (distance to decision boundary). Formally, we first show (using basically Markov's inequality) that $P(Y(w^T X + b) \leq \rho \reg) \leq \frac{1}{1 - \rho \reg} \rampL(\theta, P)$, where we recall that $r : \mathbb{R} \to [0, 1]$ is the ramp loss which is bounded between $0$ and $1$:
\begin{align*}
\rampL(\theta, P) &= \E_{X, Y \sim P}[\ramp(Y(w^TX + b))] \\
&\geq \E_{X, Y \sim P}[\ramp(Y(w^TX + b)) \mathbb{I}_{Y(w^T X + b) \leq \rho \reg}] \\
&\geq \E_{X, Y \sim P}[(1 - \rho \reg) \mathbb{I}_{Y(w^T X + b) \leq \rho \reg}] \\
&= (1 - \rho \reg) P(Y(w^T X + b) \leq \rho \reg)
\end{align*}
Here, the inequality on the third line follows because if $Y(w^T X + b) \leq \rho \reg$ where $0 < \rho \reg \leq 1$, then $\ramp(Y(w^T X + b)) \geq 1 - \rho \reg$, from the definition of the ramp loss.

This gives us:
\begin{equation} \label{eqn:marginErrorBound} P(Y(w^T X + b) \leq \rho \reg) \leq \frac{1}{1 - \rho \reg} \rampL(\theta, P) \end{equation}

The high level intuition of the next step is that since the shift is small, only points $x, y$ with $y(w^T x + b) \leq \rho \reg$ can be misclassified after the distribution shift, and from the previous step since there aren't too many of these the error of $\theta$ on $Q$ is small.

Formally, fix $\epsilon > 0$ with $\rho + \epsilon < \frac{1}{R}$, and let $f_y : \xspace \to \xspace$ be a mapping such that for all measurable $A \subseteq \mathbb{R}^d$, $P(f^{-1}(A) \mid Y=y) = Q(A | Y=y)$, with $\sup_{x \in \xspace} ||f_y(x) - x||_2 \leq \rho + \epsilon$ for $y \in \{-1, 1\}$\footnote{We need the $\epsilon$ here because a mapping with exactly the $\wasser$ distance may not exist, although if they $P$ and $Q$ have densities then such a mapping does exist.}, then we have:
\begin{align*}
&\error(\theta, Q) \\
=\; &Q(Y \neq \mbox{sign}(w^T(X+b))) \\
=\; &Q(Y(w^T(X+b)) \leq 0) \\
=\; &Q(Y=1)Q(w^T X + b \leq 0 \mid Y = 1) \; + \\
&Q(Y=-1)Q(w^T X + b \geq 0 \mid Y = -1) \\
=\; &P(Y=1)P(w^T f_1(X) + b \leq 0 \mid Y = 1) \; + \\
&P(Y=-1)P(w^T f_{-1}(X) + b \geq 0 \mid Y = -1) \\
\leq\; &P(Y=1)P(w^T X + b \leq (\rho + \epsilon)R \mid Y = 1) \; + \\
&P(Y=-1)P(w^T X + b \geq -(\rho + \epsilon)R \mid Y = -1) \\
=\; &P(Y (w^T X + b) \leq (\rho + \epsilon)R) 
\end{align*}

Where the inequality follows from Cauchy-Schwarz:
\begin{align*}
|w^T X - w^T f_y(X)| &= |w^T(X - f_y(X))|  \\
&\leq ||w||_2 ||X - f_y(X)||_2 \\
&\leq R (\rho + \epsilon)
\end{align*}

Combining this with Equation~\eqref{eqn:marginErrorBound}, this gives us:
\begin{equation} \error(\theta, Q) \leq \frac{1}{1 - (\rho + \epsilon) \reg} \rampL(\theta, P) \end{equation}
Since $\epsilon > 0$ was arbitrary, by taking the infimum over all $\epsilon > 0$, we get:
\begin{equation} \error(\theta, Q) \leq \frac{1}{1 - \rho \reg} \rampL(\theta, P) \end{equation}
Which was what we wanted to show.

\end{proof}

From the previous lemma, $M_{\theta}$ has low error on $Q$, or in other words only occasionally mislabels examples from $Q$. The next lemma says that if we minimize the ramp loss on a distribution where the points are only occasionally mislabeled, then we learn a classifier with low (good) ramp loss as well.

\begin{lemma}
\label{lem:boundMarginFromError}
Given random variables $X, Y, Y'$ (defined on the same measure space) with joint distribution $P$, where $X$ denotes the distribution over inputs, and $Y, Y'$ denote distinct distributions over labels. If $P(Y \neq Y') \leq \beta$ then for any $\theta$, $\rampL(\theta, P_X P_{Y' \mid X}) \leq \rampL(\theta, P_X P_{Y \mid X}) + \beta$. Here $P_X P_{Y \mid X}$ denotes the distribution where the input $X$ is sampled from $P_X$ and then the label is sampled from $P_{Y \mid X}$.
\end{lemma}

\begin{proof}
Let $\theta = (w, b)$. The proof is by algebra, where we recall that $r : \mathbb{R} \to [0, 1]$ is the ramp loss which is bounded between $0$ and $1$:
\begin{align*}
&\rampL(\theta, P_X P_{Y' \mid X}) \\
=\;&\E\Big[\ramp \big( Y'(w^TX + b) \big)\Big] \\
=\;& \E\Big[\ramp \big( Y'(w^TX + b) \big)\mathbb{I}_{Y = Y'}\Big] + \\
&\E\Big[\ramp \big( Y'(w^TX + b) \big) \mathbb{I}_{Y \neq Y'}\Big] \\
\leq\;& \E\Big[\ramp \big( Y'(w^TX + b) \big)\mathbb{I}_{Y = Y'}\Big] + \E\Big[\mathbb{I}_{Y \neq Y'}\Big] \\
=\;& \E\Big[\ramp \big( Y'(w^TX + b) \big)\mathbb{I}_{Y = Y'}\Big] + \beta \\
=\;& \E\Big[\ramp \big( Y(w^TX + b) \big)\mathbb{I}_{Y = Y'}\Big] + \beta \\
\leq\;& \E\Big[\ramp \big( Y(w^TX + b) \big)\Big] + \beta \\
=\;& \rampL(\theta,  P_X P_{Y \mid X}) + \beta
\end{align*}
\end{proof}

\begin{proof}[Proof of Theorem~\ref{thm:gradualSelfTrain}]
We begin by noting that there is some $\theta^* \in \linearmodels_{\reg}$ that gets low loss $\alpha^*$ on $Q$:
\begin{equation} \rampL(M_{\theta^*}, Q) = \alpha^* = \min_{\theta^* \in \linearmodels_{\reg}} \rampL(M_{\theta^*}, Q) \end{equation}

In self-training, we do not have access to labels from $Q$ so we use $M_{\theta}$ to pseudolabel examples $X$ from $Q$, so let $w, b = \theta$ and let $Y' \mid X = \sign(w^T X + b)$ be the pseudolabel distribution $Q_{Y' \mid X}$.

However, our pseudolabels are mostly correct.
That is, let $\beta = \frac{2}{1 - \rho \reg} \rampL(M_{\theta}, P)$.
Since the conditions of Lemma~\ref{lem:boundErrorFromMargin} are satisfied, $\error(M_{\theta}, Q) \leq \beta$.
This means that the pseudolabels from $M_{\theta}$ and the true labels on $Q$ mostly agree: $Q(Y \neq Y') \leq \beta$.
So by Lemma~\ref{lem:boundMarginFromError}, $\theta^*$, which attained low loss $\alpha^*$ on $Q$, also does fairly well on the pseudolabeled distribution $Q_X Q_{Y' \mid X}$, which denotes the distribution where the input $X$ is sampled from $Q_X$ and then the label is sampled from $Q_{Y' \mid X}$:
\begin{align}
\rampL(M_{\theta^*}, Q_X Q_{Y' \mid X}) &\leq \rampL(M_{\theta^*}, Q) + \beta \nonumber\\
&\leq \alpha^* + \beta
\end{align}

Since we have $n$ examples from $Q_X Q_{Y' \mid X}$, from Lemma~\ref{lem:finiteSampleBound} the empirical risk minimizer $\theta'$ on the $n$ examples satisfies:
\begin{align}
\label{eqn:ermOnPseudolabel}
\rampL(\theta', Q_X Q_{Y' \mid X}) \leq &\min_{\theta \in \linearmodels_{\reg}} \rampL(\theta, Q_X Q_{Y' \mid X}) \nonumber\\
&+\frac{4B \reg+ \sqrt{2 \log{2 / \delta}}}{\sqrt{n}}
\end{align}

But minimizing the loss on $Q_X Q_{Y' \mid X}$ explicitly gives us a lower loss than $\theta^*$ gets on $Q_X Q_{Y' \mid X}$ (recall that $\theta^*$ is the minimizer of the loss on $Q$ which is different):
\begin{align}
\label{eqn:eplicitlyMinPseudolabel}
\min_{\theta \in \linearmodels_{\reg}} \rampL(\theta, Q_X Q_{Y' \mid X}) &\leq \rampL(M_{\theta^*}, Q_X Q_{Y' \mid X}) \nonumber\\
&\leq \alpha^* + \beta
\end{align}

Combining Equations~\eqref{eqn:ermOnPseudolabel} and~\eqref{eqn:eplicitlyMinPseudolabel}, we get:
\begin{equation}
\rampL(\theta', Q_X Q_{Y' \mid X}) \leq \alpha^* + \beta + \frac{4B \reg+ \sqrt{2 \log{2 / \delta}}}{\sqrt{n}}
\end{equation}
This bounds the ramp loss of $\theta'$ on the pseudolabeled distribution $Q_X Q_{Y' \mid X}$---to convert this back to $Q$ we apply Lemma~\ref{lem:boundMarginFromError} again which we can since $Q(Y \neq Y') \leq \beta$, which gives us:
\begin{equation}
\rampL(\theta', Q) \leq \alpha^* + 2\beta + \frac{4B \reg+ \sqrt{2 \log{2 / \delta}}}{\sqrt{n}}
\end{equation}
This completes the proof.

\end{proof}

\newtheorem*{gradualSelfTrainCorollary}{Restatement of Corollary~\ref{cor:gradualSelfTrain}}

\begin{gradualSelfTrainCorollary}
\gradualSelfTrainCorollaryText{}
\end{gradualSelfTrainCorollary}

\begin{proof}
% We begin with labeled data from $P_0$, in which case from Lemma~\ref{lem:finiteSampleBound}, with probability at least $1 - \delta/T$ we get a classifier $M_{\theta_0}$ with $\theta_0 \in \linearmodels_{\reg}$ with loss:
% \begin{equation} \rampL(M_{\theta_0}, P_0) \leq \alpha^* + \frac{4B \reg + \sqrt{2 \log{2T / \delta}}}{\sqrt{n}} \end{equation}
We begin with a classifier with loss $\alpha_0$.
Applying Theorem~\ref{thm:gradualSelfTrain} for each subsequent step of self-training, letting $\beta = \frac{2}{1 - \rho \reg}$, we get:
\begin{align}
\rampL(M_{\theta_{i+1}}, P_{i+1}) \leq &\beta\rampL(M_{\theta_i}, P_i) + \alpha^* \nonumber\\
&+ \frac{4B \reg + \sqrt{2 \log{2T / \delta}}}{\sqrt{n}}
\end{align}
Expanding, this becomes the sum of a geometric series. Noting that $\alpha^* \leq \alpha_0$, by using the formula for the sum of geometric series, we get:
\begin{equation} 
\rampL(M_{\theta_T}, P_T) \leq \beta^{T+1}\Big(\alpha_0 + \frac{4B \reg + \sqrt{2 \log{2T / \delta}}}{\sqrt{n}}\Big)
\end{equation}
\end{proof}

\newtheorem*{selfTrainingExponentialExample}{Restatement of Example~\ref{ex:selfTrainingExponential}}

\begin{selfTrainingExponentialExample}
\selfTrainingExponentialText{}
\end{selfTrainingExponentialExample}

\begin{proof}
The construction works even in 1-D. We will consider regularized linear models $\linearmodels_{\reg}$ with $\reg = 1$, so $\rho < \frac{1}{\reg}$. Such a model in 1D can be parametrized by 2 parameters, $w, b \in \mathbb{R}$ with $|w| \leq 1$, where the output of the linear model for an input $x \in \mathbb{R}$ is $wx + b$, and the label is $\sign(wx + b)$.

First we give intuition, and then we dive into the formal details of the construction.

We start with a classifier $\theta_0 = (w_0, b_0) = (1, 0)$.
We will construct the distributions so that the classifier $\theta_t = \theta_0$ for all $t$, that is, gradual self-training will not update the classifier.
In the initial distribution $P_0$, all the negative examples will be located at $x = -10$, so the classifier gets them correct and incurs 0 loss on them.
$\alpha_0$ fraction of the positive examples will be at $x = -0.1$, these examples are misclassified so the classifier incurs loss $\alpha_0$.
The rest of the positive examples will be at $x = 1$, and the classifier incurs $0$ loss on them.

In distribution $P_1$, $0.5 \alpha_0$ fraction of the positive examples will move from $x = 1$ to $x = 0.5$, but everything else stays the same as in $P_0$.
After pseudolabeling and self-training, the classifier still stays the same, that is $\theta_1 = \theta_0$.
This is because the $\alpha_0$ fraction of examples at $x=-0.1$ will be pseudolabeled negative, the $0.5\alpha_0$ fraction of examples at $x = 0.5$ pseudolabeled positive, and the remaining positive examples at $x=1$ will be pseudolabeled positive. 
Training on this pseudolabeled distribution gives us $\theta_1 = (w_1, b_1) = (1, 0)$ as the optimal parameters.

In $P_2$, the $0.5 \alpha_0$ fracton of points at $x = 0.5$ moves to $x = -0.1$.
After pseudolabeling and self-training, we still get $\theta_2 = \theta_0$.
At this point the classifier incurs loss $1.5 \alpha_0$.
We repeat this process, except for $P_3$, $\alpha_0$ fraction of the positive examples move from $x = 1$ to $x = 0.5$, and then the next time in $P_5$, $2 \alpha_0$ fraction of the positive examples move from $x = 1$ to $x = 0.5$, etc.
So in this way the loss grows exponentially.

We now give the formal construction, which works even in just 1 dimension.
First, we choose $S$ to be the maximum integer such that $(2^{S-1} + \frac{1}{2}) \alpha_0 < \frac{1}{2}$.
We have $S \geq 1$, because $(2^{1-1} + \frac{1}{2}) \alpha_0 = \frac{3}{2} \alpha_0 \leq \frac{3}{2} \frac{1}{4} < \frac{1}{2}$.

We now define a sequence of weights, which represents the fraction of points we move in each step as in the sketch above.
For $0 \leq i \leq S-1$, let $w_i = \frac{1}{2} 2^i \alpha_0$, and let $w_S = \frac{1}{2} - (2^{S-1} + \frac{1}{2}) \alpha_0$.
From the sum of geometric series, we can verify that each of these weights are positive, and the weights sum up to $\frac{1}{2}$.

We now define the distributions at each step, we case on whether the step is odd or even since as in the above high level sketch, it takes 2 steps to move a point from $x=1$ across to the other side of the decision boundary.
One subtlety is that unlike the sketch above, since we use the Monge form of the Wasserstein distance, we cannot have all the points exactly at $x=1$ but keep them separated by a small distance $\delta = \frac{1}{10S}$.
This is a technical detail, so on a first reading the reader may just pretend $\delta=0$ to work through the structure of the proof.

(Odd case) For $0 \leq t < \min(T, S+1)$, $P_{2t+1}$ is given by:
\begin{align}
&P_{2t+1}(x = -10 \land Y = -1) = 0.5 \\
&P_{2t+1}(x = -0.1 \land Y = 1) = \alpha_0 + \sum_{i=0}^{t-1} w_i \\
&P_{2t+1}(x = 0.5 \land Y = 1) = w_t \\
& P_{2t+1}(x = 1+i\delta \land Y = 1) = w_i \quad \forall t < i \leq S
\end{align}

(Even case) For $0 \leq t \leq \min(T, S+1)$, $P_{2t}$ is given by:
\begin{align}
&P_{2t}(x = -10 \land Y = -1) = 0.5 \\
&P_{2t}(x = -0.1 \land Y = 1) = \alpha_0 + \sum_{i=0}^{t-1} w_i \\
& P_{2t}(x = 1+i\delta \land Y = 1) = w_i \quad \forall t \leq i \leq S
\end{align}

If $T \geq S+1$, then for $2S+2 \leq i \leq 2T$, we set $P_i = P_{2S+2}$ (by this step the classifier will have reached ramp loss and error 0.5).

We can check that if $t < 2S$, then the classifier obtained from gradual self-training is $w_t = (1, 0)$ (the classifier does not change after self-training).
When $t = 2S$, $w_t = (1, -0.9)$, and finally if $2S < t$ then $w_t = (1, -1.5)$.
The edge case is because at the end all positive points are to the left of the classifier, so the classifier moves to the right.

Next we examine the loss values.
If $t \leq S$, the fraction of examples the classifier $w_{2t}$ gets wrong on $P_{2t}$ is:
\begin{equation} \alpha_0 + \sum_{i=0}^{t-1} w_i = (2^{t-1} + \frac{1}{2}) \alpha_0 \geq \frac{1}{2} 2^t \alpha_0 \end{equation}
If $t > S$, the fraction of examples the classifier $w_{2t}$ gets wrong on $P_{2t}$ is $0.5$.
The ramp loss is bounded below by the error rate, which means:
\begin{equation} \rampL(\theta', P_{2T}) \geq \min(0.5, \frac{1}{2} 2^T \alpha_0) \end{equation}
As desired.

We can verify that for every $i$, $\wasser(P_i, P_{i+1}) \leq 0.6 < 1 = \frac{1}{\reg}$, so these distributions satisfy the \gradShiftAssump{} assumption.
$P_i(Y = 1) = P_i(Y = -1) = 0.5$ for all $i$, so the distributions satisfy the \noLabShiftAssump{} assumption.
The classifier $(w, b) = (1, 5)$ gets $0$ loss on all $P_i$, so the distributions satisfy the \sepAssump{} assumption with $\alpha^* = 0$.
Finally, the data is all bounded in a constant region, between $x = -10$ and $x = 2$, so the distributions satisfy the \boundedAssump{} assumption.
\end{proof}

\newtheorem*{noRegularizationNoGainExample}{Restatement of Example~\ref{ex:noRegularizationNoGain}}

\begin{noRegularizationNoGainExample}
\noRegularizationNoGainText{}
\end{noRegularizationNoGainExample}

\begin{proof}
The proof is straightforward: scaling up the parameters of the original model $\theta$ gives us a $\theta'$ that gets $0$ loss (ramp or hinge) on the pseudolabeled distribution, but does not change the model predictions.
For simplicity, we focus on the ramp loss but the proof applies to the hinge loss as well.
Suppose $\theta = (w, b)$, where $w \in \mathbb{R}^d$ and $b \in \mathbb{R}$.

We choose our new parameters to be $\theta' = (\alpha w, \alpha b)$, where $\alpha \geq 1$ is a scaling factor we will choose.
Then we can write $L(\theta')$, the loss of $\theta'$ on the pseudolabeled examples $S$ as:
\begin{align}
L(\theta') &= \frac{1}{|S|} \sum_{x \in S} \rampl(M_{\theta'}(x), \sign(M_{\theta}(x))) \nonumber\\
&= \frac{1}{|S|} \sum_{x \in S} \ramp(\sign(w^T x + b) (\alpha w^T x + \alpha b)) \nonumber\\
&= \frac{1}{|S|} \sum_{x \in S} \ramp(\alpha \lvert w^T x + b \rvert )
\end{align}
Now, we can choose large enough $\alpha$ so that the term inside the $\ramp$ in the last line above is always $\geq 1$:
\begin{equation} \alpha = \frac{1}{\min_{x \in S}{\lvert w^Tx + b \rvert}} \end{equation}
So now, $\lvert (w')^Tx + b' \rvert = \alpha \lvert (w^Tx + b) \rvert \geq 1$ for all $x \in S$.
This gives us that $L(\theta') = 0$, since $\ramp(m) = 0$ for $m \geq 1$.
Note that this is true for the hinge loss as well, $\hinge(m) = 0$ for $m \geq 1$.
Since $L$ is bounded below by $0$, $\theta'$ is a minimizer of the loss on the pseudolabeled distribution (which is what self-training minimizes, see Equation~\eqref{eqn:selfTrainSample}).

Since $\theta'$ is just a scaled up version of $\theta$, it does not change the predictions:
\begin{equation} \sign(\alpha w^T x + \alpha b) = \sign(w^T x + b) \end{equation}

\end{proof}

\newtheorem*{softLabelsBadExample}{Restatement of Example~\ref{ex:softLabelsBad}}

\begin{softLabelsBadExample}
\softLabelsBadText{}
\end{softLabelsBadExample}

\begin{proof}
The reason for this is that the logistic loss is a proper scoring loss---if we fix $p$, the loss of $\logl(p, p')$ is minimized when $p' = p$. That is, if $0 \leq p, p' \leq 1$:
\begin{equation} \logl(p, p) \leq \logl(p, p') \end{equation}
So we have:
\begin{align}
\logL(\theta') &= \E[\logl(\sigma(M_{\theta}(X)), \sigma(M_{\theta'}(X)))] \nonumber\\
&\geq \E[\logl(\sigma(M_{\theta}(X)), \sigma(M_{\theta}(X)))] \nonumber\\
&= \logL(\theta)
\end{align}
\end{proof}

\newtheorem*{hingeLossBadExample}{Restatement of Example~\ref{ex:hingeLossBad}}

\begin{hingeLossBadExample}
\HingeLossBadText{}
\end{hingeLossBadExample}

\begin{proof}

We construct an example in 2D.
We consider the set of regularized linear models $\linearmodels_{\reg}$, where $\reg = 1$.
Such a classifier is parametrized by $(w, b)$ where $w \in \mathbb{R}^2$ with $||w||_2 \leq 1$, and $b \in \mathbb{R}$.
The output of the model is $M_{w, b}(x) = w^Tx+b$, and the predicted label is $\sign(w^Tx + b)$.

Set $\alpha_0 = \min(\frac{1}{2}, \frac{2\alpha}{3})$.
We will construct an example where the initial hinge error is $\leq \alpha_0$, but it increases to over $1$ and gets every example wrong, in 2 distribution shifts, even though there exists a single classifier with $0$ hinge loss across all the distributions.
Let $w_0 = (1,0)$ and $b_0 = 0$.
Consider a distribution $Q_{\delta}$, for $\delta \in \mathbb{R}$, defined as follows:
\begin{align}
&Q_{\delta}(Y = 1 \wedge X = (\delta, 1)) = \frac{1 - \alpha_0}{2} && [\text{Point 1}] \nonumber\\
&Q_{\delta}(Y = 1 \wedge X = (-\frac{1}{2}, \frac{1-\alpha_0}{\alpha_0})) = \frac{\alpha_0}{2} && [\text{Point 2}] \nonumber\\
&Q_{\delta}(Y = -1 \wedge X = (-\delta, -1)) = \frac{1 - \alpha_0}{2} && [\text{Point 3}] \nonumber\\
&Q_{\delta}(Y = -1 \wedge X = (\frac{1}{2}, -\frac{1-\alpha_0}{\alpha_0})) = \frac{\alpha_0}{2} && [\text{Point 4}] \nonumber
\end{align}
We will set $P_0 = Q_{1}$, $P_1 = Q_{1/3}$, and $P_2 = Q_{-1/3}$.
First, we note that the Wasserstein-infinity distance between any consecutive one of these is at most $2/3 < 1$.

Next, we can verify that $\hingeL(w_0, P_0) = \frac{3}{2}\alpha_0 \leq \alpha$.
In particular, $w_0$ gets points 2 and 4 incorrect, and points 1 and 3 correct with margin 1.
Computing the expectation of the loss, we get $\frac{3}{2}\alpha_0$.

Now the algorithm self-trains on $P_1$: $w_0$ pseudolabels points 1 and 4 positive ($y = 1$), and pseudolabels points 2 and 3 negative ($y = -1$), again getting points 2 and 4 incorrect.
From the KKT conditions, we can verify that the minimizer of the hinge loss on these pseudolabeled points is $w_1 = w_0$, and $b_2 = 0$.

Finally, the algorithm self-trains on $P_2$: here $w_0$ pseudolabels points 3 and 4 positive, and 1 and 2 negative.
That is, it gets all the examples wrong.
Self-training on these pseudolabels, the model still gets every example wrong (one solution is $w_2 = (0, -1)$ and $b_2 = 0$).
So $\error(w_2, P_2) = 1$, and the hinge loss is lower bounded by the error with $\hingeL(w_2, P_2) \geq \error(w_2, P_2)$.

On the other hand, the classifier $w^* = (0, 1)$ and $b^* = 0$, gets hinge loss $0$ on $P_1, P_2, P_3$.

\end{proof}

\newtheorem*{selfTrainingNoShiftBoundProp}{Restatement of Proposition~\ref{prop:selfTrainingNoShiftBound}}

\begin{selfTrainingNoShiftBoundProp}
\selfTrainingNoShiftBoundText{}
\end{selfTrainingNoShiftBoundProp}

\newcommand{\unlabeledL}{\ensuremath{U_r}}
We give intuition for our argument, and then dive into the formal proof.
Suppose we start out with a model that has ramp loss $\alpha_0$ on $P = P_0 = \cdots = P_T$.
After a single step of self-training, the loss can increase to $2\alpha_0$ on $P$.
So a naive argument leads to an exponential bound (since the loss is now $2\alpha_0$, it can increase to $2 \cdot 2\alpha_0$ after another round of self-training, etc, so after $T$ steps the loss on $P$ is bounded by $2^T \alpha_0$).
Showing a linear upper bound requires a more subtle argument that tracks some other invariants, and not just the loss value.

Roughly speaking, if the initial loss is below $\alpha_0$, there cannot be more than $\alpha_0$ fraction of points near the decision boundary.
We show that this invariant is maintained by self-training: the `number' of points near the decision boundary decreases, so it always stays below the initial value $\alpha_0$.
Finally, we show that if there are $\alpha_0$ points near the decision boundary, then self-training cannot increase the loss by more than $\alpha_0$ \emph{no matter what the current loss is}.
This shows that at each step the loss can only increase by $\alpha_0$.
Compare this with Example~\ref{ex:selfTrainingExponential}, where we do have distribution shift---in this case the `number' of points near the decision boundary can keep increasing which can lead to an exponential growth in the loss.

We now dive into the formal proof---we begin by making some definitions and stating and proving lemmas that formalize the above intuition.

In self-training, we pseudolabel an example $x$ with label $\sign(M_{\theta}(x))$.
We define the corresponding distribution on the pseudolabels $P_{Y \mid x, \theta}$ by $Y \mid x, \theta = \sign(M_{\theta}(x))$.

Recall that the loss of $\theta$ on labeled data is (where $\ramp$ is the ramp loss):
\begin{align}
\rampL(\theta, P) &= \E_{X, Y \sim P}[\rampl(M_{\theta}(X), Y)] \nonumber\\
&= \E_{X, Y \sim P}[\ramp(Y M_{\theta}(X))]
\end{align}

We define a loss on unlabeled data which corresponds to the loss of $\theta$ if every example was labeled by $M_{\theta}$.
This roughly corresponds to the `number' of points near the decision boundary, since points far from the decision boundary incur 0 loss, but points near the decision boundary incur a loss between 0 and 1.
Note that the unlabeled loss does not use the labels $Y$.
Letting $P_X$ denote the marginal distribution of $P$ on $X$, and $P_X P_{Y \mid X}$ denote the distribution where $X$ is sampled from $P_X$ and $Y$ is sampled from $P_{Y \mid X}$, the unlabeled loss $\unlabeledL$ is:

\begin{align}
\unlabeledL(\theta, P)
&= \rampL(\theta, P_X P_{Y \mid X, \theta}) \nonumber\\
&= \E_{X \sim P}[\rampl(M_{\theta}(X), \sign(M_{\theta}(X)) )] \nonumber\\
&= \E_{X \sim P}[\ramp(\lvert M_{\theta}(X) \rvert)] 
\end{align}

The unlabeled loss $\unlabeledL$ and labeled loss $\rampL$ are always defined since the ramp loss is bounded below by $0$.
A straightforward lemma shows that the unlabeled loss lower bounds the labeled loss.

\begin{lemma}[Lower bounds labeled loss]
\label{lem:unlabeled_lower_bounds_labeled}
The unlabeled loss lower bounds the labeled loss: $\unlabeledL(\theta, P) \leq \rampL(\theta, P)$.
\end{lemma}

\begin{proof}
Since $Y \in \{-1, 1\}$, 
\begin{equation}
\lvert M_{\theta}(X) \rvert = \lvert Y M_{\theta}(X) \rvert \geq Y M_{\theta}(X)
\end{equation}
Now, since $\ramp$ is a non-increasing function, we have:
\begin{equation}
\ramp(\lvert M_{\theta}(X) \rvert) \leq \ramp(Y M_{\theta}(X))
\end{equation}
Taking expectations on both sides:
\begin{equation}
\unlabeledL(\theta, P) \leq \rampL(\theta, P)
\end{equation}
\end{proof}

The next lemma shows that each step of self-training decreases the unlabeled loss.

\begin{lemma}[Unlabeled loss decreases]
\label{lem:unlabeled_loss_decreases}
If $\theta, \theta' \in \Theta$ and $\theta' = \selftrain(\theta, P)$, then $\unlabeledL(\theta', P) \leq \unlabeledL(\theta, P)$.
\end{lemma}

\begin{proof}
Since the unlabeled loss does not depend on the labels, we have:
\begin{equation}
\unlabeledL(\theta', P) = \unlabeledL(\theta', P_X P_{Y \mid X, \theta})
\end{equation}
From Lemma~\ref{lem:unlabeled_lower_bounds_labeled}, the unlabeled loss lower bounds the labeled loss:
\begin{equation}
\unlabeledL(\theta', P_X P_{Y \mid X, \theta}) \leq \rampL(\theta', P_X P_{Y \mid X, \theta})
\end{equation}
But $P_X P_{Y \mid X, \theta}$ is the distribution of pseudolabels produced by model $\theta$, which is exactly what self-training ($\theta' = \selftrain(\theta, P)$) minimizes (recall the definition of self-training in Equation~\eqref{eqn:selfTrainPop}), so $\theta'$ has lower loss than $\theta$ on the pseudolabeled distribution:
\begin{equation}
\label{eqn:st_loss_bounds_unlabeled}
\rampL(\theta', P_X P_{Y \mid X, \theta}) \leq \rampL(\theta, P_X P_{Y \mid X, \theta}) = \unlabeledL(\theta, P)
\end{equation}
Which means that:
\begin{equation}
\unlabeledL(\theta', P) \leq \unlabeledL(\theta, P)
\end{equation}
\end{proof}

We now show a type of triangle inequality for the loss, which says that the loss of $\theta'$ on $P$ is upper bounded by the loss of $\theta'$ on pseudolabels from $\theta$ plus the loss of $\theta$ on $P$.

\begin{lemma}[Triangle Inequality]
\label{lem:unlabeled_triangle_inequality}
$\rampL(\theta', P) \leq \rampL(\theta', P_X P_{Y \mid X, \theta}) + \rampL(\theta, P)$
\end{lemma}

\begin{proof}

We will first show that for any $x$, $y$, $\theta$, $\theta'$:
\begin{align}
\rampl(M_{\theta'}(x), y) \leq \max(&\rampl(M_{\theta'}(x), \sign(M_{\theta}(x))), \nonumber\\
&\rampl(M_{\theta}(x), y))
\end{align}
We can prove this by casing. If $M_{\theta'}(x)$ and $M_{\theta}(x)$ have different signs, or $M_{\theta}(x)$ and $y$ have different signs, then the RHS is 1.
But the ramp loss is bounded above by 1, so the LHS has loss at most 1, which makes this statement true.
Otherwise, suppose $M_{\theta'}(x)$, $M_{\theta}(x)$, and $y$ all have the same signs---but then $\sign(M_{\theta}(x)) = y$, so $\rampl(M_{\theta'}(x), y) = \rampl(M_{\theta'}(x), \sign(M_{\theta}(x)))$.

With this in hand, the result follows with some algebra:
\begin{align}
&\rampL(\theta', P) \nonumber\\
= &\E[\rampl(M_{\theta'}(X), Y)] \nonumber\\
\leq & \E[\max(\rampl(M_{\theta'}(X), \sign(M_{\theta}(X))), \rampl(M_{\theta}(X), Y))] \nonumber\\
\leq & \E[\rampl(M_{\theta'}(X), \sign(M_{\theta}(X))) + \rampl(M_{\theta}(X), Y)] \nonumber\\
= &\rampL(\theta', P_X P_{Y \mid X, \theta}) + \rampL(\theta, P)
\end{align}
\end{proof}

Next, we show that if the unlabeled loss of $\theta$ is less than $\alpha$, then self-training cannot increase the loss by more than $\alpha$.

\begin{lemma}[Upper bounding loss growth]
\label{lem:loss_growth_upper_bound}
Suppose $\unlabeledL(\theta, P) \leq \alpha$ and let $\theta' = \selftrain(\theta, P)$. Then:
\[ \rampL(\theta', P) \leq \rampL(\theta, P) + \alpha \]
\end{lemma}

\begin{proof}
From Lemma~\ref{lem:unlabeled_triangle_inequality}, it suffices to show that $\rampL(\theta', P_X P_{Y \mid X, \theta}) \leq \alpha$. But as in Equation~\eqref{eqn:st_loss_bounds_unlabeled}, this is simply because $\theta'$ minimizes the pseudolabeled loss so we have:
\begin{equation}
\rampL(\theta', P_X P_{Y \mid X, \theta}) \leq \unlabeledL(\theta, P) \leq \alpha
\end{equation}
\end{proof}

The proof of Proposition~\ref{prop:selfTrainingNoShiftBound} now simply inductively applies Lemma~\ref{lem:unlabeled_loss_decreases} and Lemma~\ref{lem:loss_growth_upper_bound}.

\begin{proof}[Proof of Proposition~\ref{prop:selfTrainingNoShiftBound}]
Let $\theta_t = \selftrain(\theta_{t-1}, P_t)$ for $1 \leq t \leq T$.
The unlabeled loss lower bounds the labeled loss: that is, since $\rampL(\theta_0, P_0) \leq \alpha_0$, from Lemma~\ref{lem:unlabeled_lower_bounds_labeled}, $\unlabeledL(\theta_0, P_0) \leq \alpha_0$.
The unlabeled loss can only decrease with self-training: that is, inductively applying Lemma~\ref{lem:unlabeled_loss_decreases}, we get that for all $t$, $\unlabeledL(\theta_t, P_t) \leq \alpha_0$.
Then from Lemma~\ref{lem:loss_growth_upper_bound}, the loss can only increase by $\alpha_0$ at each step of self-training, so $\rampL(\theta_T, P_T) \leq \rampL(\theta_0, P_0) + \alpha_0 T \leq \alpha_0 (T + 1)$.
\end{proof}

The next Example shows that even without distribution shift, self-training can increase the loss of a model from $\alpha_0$ to nearly $2\alpha_0$.

\begin{example}
Even under the \sepAssump{} and \boundedAssump{} assumptions, for every $0.25 > \alpha_0 > \epsilon > 0$, there exists a model $\theta_0$ and distribution $P$ with $\rampL(\theta_0, P) \leq \alpha_0$ but $\rampL(\selftrain(\theta_0, P), P) \geq 2 \alpha_0 - \epsilon$. 
\end{example}

\begin{proof}
We give an example in 1D, where a linear model can be parametrized by 2 parameters, $w, b \in \mathbb{R}$ with $|w| \leq 1$, where the output of the linear model for an input $x \in \mathbb{R}$ is $wx + b$, and the label is $\sign(wx + b)$.

Let $\delta = \epsilon/3$ and $a = \alpha_0 / (1 + \delta)$.
Let the data distribution $P$ be given by:
\begin{equation} P(X = -10 \land Y = -1) = 0.5 \end{equation}
\begin{equation} P(X = 0 \land Y = 1) = a \end{equation}
\begin{equation} P(X = 1 \land Y = 1) = a - \delta \end{equation}
\begin{equation} P(X = 10 \land Y = 1) = 0.5 - 2a + \delta \end{equation}
Note that the probabilities are all non-negative and add up to $1$ and the data is bounded between $x=-10$ and $x=10$.

Let the initial model be $w_0 = 1$ and $b_0 = -\delta$.
The initial loss is $\rampL((w_0, b_0), P) = a + (a - \delta)\delta = \alpha_0 - \delta^2 \leq \alpha_0$.
We can check that after self-training, the updated parameters are $w_1 = 1$ and $b_1 = 1$.
The final loss is $\rampL((w_1, b_1), P) = 2a - \delta \geq 2\alpha_0(1 - \delta) - \delta \geq 2\alpha_0 - 3\delta = 2\alpha_0 - \epsilon$.

\end{proof}

\newpage

\section{Proofs for Section~\ref{sec:gaussian_theory}}
\label{sec:appendix_gaussian_theory}

We prove Theorem~\ref{thm:gaussian} in Section~\ref{sec:margin_theory}, following the sketch described in the paper.
Our first lemma shows that if $\mu$ does not change too much, then the optimal parameters $w^*(\mu)$ do not change too much either.

\begin{lemma}
\label{lem:lipschitzGaussian}
$w^*$ is $\frac{1}{B}$-Lipschitz, that is if $||\mu||_2, ||\mu'||_2 \geq B > 0$, then:
\begin{equation} ||w^*(\mu') - w^*(\mu)||_2 \leq \frac{1}{B} ||\mu' - \mu||_2 \end{equation}
\end{lemma}

\begin{proof}
Recall that $w^*(\mu) = \mu / ||\mu||_2$, which is well defined since $||\mu||_2 > 0$.
We will first prove that if $||v'||_2 \geq ||v||_2 = 1$, then the claim holds, that is:
\begin{equation} \label{eqn:lipschitzCase1} ||\frac{v'}{||v'||_2} - v||_2^2 \leq ||v' - v||_2^2 \end{equation}
Expanding both sides, this is equivalent to showing:
\begin{equation} 1 + ||v||_2^2 - \frac{2v^Tv'}{||v'||_2} \leq ||v'||_2^2 + ||v||_2^2 - 2v^Tv \end{equation}
Subtracting both sides by $||v||_2^2$, it suffices to show:
\begin{equation} 1 - \frac{2v^Tv'}{||v'||_2} \leq ||v'||_2^2 - 2v^Tv \end{equation}
But since $||v'|| \geq 1$, we can bound the LHS above if we multiply by $||v'||_2$:
\begin{align}
1 - \frac{2v^Tv'}{||v'||_2} &\leq ||v'||_2 - 2v^Tv' \nonumber\\
&\leq ||v'||_2^2 - 2v^Tv'
\end{align}
So Equation~\eqref{eqn:lipschitzCase1} is true.

Now we prove the main claim.
Without loss of generality, suppose $||\mu'||_2 \geq ||\mu||_2$, otherwise we can swap $\mu$ and $\mu'$.
Then we can scale $\mu'$ and reduce to the previous case:
\begin{align}
||w^*(\mu') - w^*(\mu)||_2 &= ||\frac{\mu'}{||\mu'||_2} - \frac{\mu}{||\mu||_2}||_2 \nonumber\\
&= ||\frac{\mu' / ||\mu||_2}{||\mu'||_2 / ||\mu||_2} - \frac{\mu}{||\mu||_2}||_2 \nonumber\\
&= ||\frac{(\mu' / ||\mu||_2)}{||(\mu' / ||\mu||_2)||_2} - \frac{\mu}{||\mu||_2}||_2 \nonumber\\
&\leq ||\frac{\mu'}{||\mu||_2} - \frac{\mu}{||\mu||_2}||_2 \nonumber\\
&= \frac{1}{||\mu||_2}||\mu' - \mu||_2 \nonumber\\
&\leq \frac{1}{B} ||\mu' - \mu||_2
\end{align}
Where in the inequality on the 4th line we applied Equation~\eqref{eqn:lipschitzCase1}. This completes the proof.

\end{proof}

We now state a standard lemma in measure theory, which says that if $f(x) \geq g(x)$ for all $x$, and the inequality is \emph{strict} on a set of non-zero measure (volume), then the integral of $f$ is strictly greater than the integral of $g$.

\begin{lemma}
\label{lem:basicExpectationBound}
Let $\mu$ be a measure on $\mathbb{R}^d$, and $C$ be measurable with $\mu(C) > 0$.
Suppose $f(x) > g(x)$ if $x \in C$, and $f(x) \geq g(x)$ for all $x \in \mathbb{R}^d$, where $f$ and $g$ are measurable functions with finite integrals.
Then:
\begin{equation} \int_{\mathbb{R}^d} f(X) d\mu > \int_{\mathbb{R}^d} g(X) d\mu \end{equation}
\end{lemma}

Our next lemma is the key step of the proof.
We show that $w^*(\mu)$ is a strict local minimizer of $U(w, P_{\mu, \sigma})$, that is it has lower loss than any other $w$ nearby.

\begin{lemma}
\label{lem:localMinGaussian}
For all $w \in \mathbb{R}^d$ with $||w||_2 \leq 1$ and $||w - w^*(\mu)||_2 < 1$, with $w \neq w^*(\mu)$, we have:
\begin{equation} U(w^*(\mu), P_{\mu, \sigma}) < U(w, P_{\mu, \sigma}) \end{equation}
\end{lemma}

\begin{proof}
Denote $w^*(\mu)$ as $w^*$.
By Cauchy-Schwarz, since $\|w^*\|_2 = 1$ and $\|w - w^*\|_2 < 1$, we have $w \cdot w^* > 0$, and $\|w\|_2 > 0$.
This is because $w \cdot w^* = \|w^*\|_2 + (w - w^*) \cdot w^* \geq 1 - \|w - w^*\|_2 \|w^*\|_2 > 0$.
Since the dot product is non-zero, neither vector can be $0$.

We begin by noting that $U(w, P_{\mu, \sigma})$ is well-defined and finite: because $\phi(|X|)$ is between $0$ and $1$ so the expectation is well-defined with finite, non-negative value.

\textbf{Step 1 (Scaling Parameters)}: First, we show that scaling up the parameters decreases the loss: for any $w$ and $\lambda > 1$, $U(\lambda w, P_{\mu, \sigma}) < U(w, P_{\mu, \sigma})$.

Since $\phi$ is non-increasing, $\phi(|\lambda w^T x|) \geq \phi(|w^T x|)$.
Now, let $C = \{ x \in \mathbb{R}^d : 0 < \lambda w^Tx < 1 \}$.
Since $\phi$ is strictly decreasing on $[0, 1]$, for $x \in C$, $\phi(|\lambda w^T x|) < \phi(|w^T x|)$.
$P_{\mu, \sigma}(C) > 0$ (the Gaussian mixture distribution assigns positive probability to any set with non-zero volume / Lebesgue measure).
Then from Lemma~\ref{lem:basicExpectationBound}:
\begin{equation} \E_{X \sim P_{\mu, \sigma}}[\phi(|\lambda w^T X|)] < \E_{X \sim P_{\mu, \sigma}}[\phi(|w^T X|)] \end{equation}
Which is precisely saying $U(\lambda w, P_{\mu, \sigma}) < U(w, P_{\mu, \sigma})$.

This lets us assume, without loss of generality, that $||w||_2 = 1$ since scaling up $w$ strictly decreases the loss, and the theorem statement assumes $||w||_2 \leq 1$.

\textbf{Step 2 (Rotating parameters)}: Note that rotating the entire space does not change the loss values, formally if $A$ is a rotation matrix then:
\begin{equation} U(Aw, P_{A\mu, \sigma}) = U(w, P_{\mu, \sigma}) \end{equation}
So without loss of generality, we rotate the setup so that $w$ and $w^*$ lie on the $XY$ plane (except for the first two coordinates, all coordinates are $0$).
Let $v$ be the unit bisector of $w$ and $w^*$, given by $v = (w + w^*) / ||w + w^*||_2$.
Without loss of generality, rotate the setup so that $v$ is along the positive $Y$ axis (the second coordinate is $1$, and all other coordinates are $0$), and the first two coordinates of $w^*$ are positive. 
Let $\mu = (r, s, 0)$ where $0 \in \mathbb{R}^{d-2}$, we then have that $r, s > 0$ since $w^*$ and $\mu$ are in the same direction.

\textbf{Step 3 (Symmetry argument)}: Now consider any point $u = (x, y, z)$ with $z \in \mathbb{R}^{d-2}$, with $x, y > 0$.
Consider its reflection point around $v$, $u' = (-x, y, z)$.
Let $\Delta(w, w', u) = \phi(|(w')^Tu|) - \phi(|w^Tu|)$ denote the increase in loss on $x$ from using classifier $w'$ instead of $w$.
Now, from the way we constructed $u'$, $w^T u = (w^*)^T u'$, and $(w^*)^T u = w^T u'$.
So $\Delta(w, w^*, u) = - \Delta(w^*, w, u')$.
That is, as per our sketch, the loss for $u$ decreases when using $w^*$ instead of $w$, but increases for $u'$ when using $w^*$ instead of $w$, but the magnitudes of these two quantities are equal.

Next, we will show that the probability density is higher for $u$ than $u'$.
Let $p$ denote the density of $P_{\mu, \sigma}$.
$P_{\mu, \sigma}$ is the mixture of two Gaussians, so for normalizing constant $k > 0$, we have:
\begin{align} p(u) = k\Big[&\exp\Big(-\frac{1}{2\sigma} ((r - x)^2 + (s - y)^2 + z^2) \Big) + \nonumber\\
 &\exp\Big(-\frac{1}{2\sigma} ((r + x)^2 + (s + y)^2 + z^2) \Big) \Big]
\end{align}
\begin{align}
p(u') = k\Big[&\exp\Big(-\frac{1}{2\sigma} ((r + x)^2 + (s - y)^2 + z^2) \Big) + \nonumber\\
&\exp\Big(-\frac{1}{2\sigma} ((r - x)^2 + (s + y)^2 + z^2) \Big) \Big]
\end{align}

We now use strict convexity of $\exp(-x)$ to show that $p(u) > p(u')$.
Let $a = (r - x)^2 + (s - y)^2 + z^2$, $b = (r + x)^2 + (s + y)^2 + z^2$, $c = (r+x)^2 - (r - x)^2 = 4rx$. Since, $x, y, r, s > 0$, we have $0 < a < b$ and $0 < c < b-a$. Letting $f(x) = \exp(-x/(2\sigma))$ we can rewrite the above probabilities as:
\begin{equation} p(u) = k\Big[f(a) + f(b) \Big] \end{equation}
\begin{equation} p(u') = k\Big[f(a+c) + f(b-c) \Big] \end{equation}
Finally, we use strict convexity of $f(x) = \exp(-x)$ to show the desired result. Since $a < a + c < a + b$, for some $\alpha \in (0, 1)$, we can write:
\begin{equation} a + c = \alpha a + (1 - \alpha)b \end{equation}
\begin{equation} b - c = (1 - \alpha)a + \alpha b \end{equation}
Then, from strict convexity, we have:
\begin{equation} f(a + c) < \alpha f(a) + (1 - \alpha)f(b) \end{equation}
\begin{equation} f(b - c) < (1 - \alpha)f(a) + \alpha f(b) \end{equation}
Adding both of these, we get:
\begin{equation} f(a+c) + f(b-c) < f(a) + f(b) \end{equation}
That is, we have shown $p(u) > p(u')$.

The case when $u = (-x, -y, z)$, where $z \in \mathbb{R}^{d-2}$ and $x, y > 0$ is symmetric.
We ignore points $\{(x, y, z) : x = 0 \vee y = 0 \}$ since this has measure 0.

\textbf{Step 4 (Expectation)}:
We give intuition and then dive into the math.
For every pair of points in our pairing in Step 3, the contribution to the loss of $w^*$ is at most as high as the contribution to the loss of $w$.
So this trivially gives us $L(w^*, P_{\mu, \sigma}) \leq L(w, P_{\mu, \sigma})$, but we want a strict inequality.
However, we can find a set of points with non-zero volume (Lebesgue measure) where the contribution to the loss for $w^*$ is strictly less than for $w$, which completes the proof.

Formally, letting $S_{+} = \{ (x, y, z) : z \in \mathbb{R}^{d-2}, x > 0, y > 0 \}$, we can write (where we defined $\Delta$ in Step 3):
\begin{align}
&L(w, P_{\mu, \sigma}) - L(w^*, P_{\mu, \sigma}) \nonumber\\
=\;& 2\int_{S_{+}} \big[ p(u) \Delta(w^*, w, u) + p(u') \Delta(w^*, w, u') \big]
\end{align}
Where the $2$ comes from the fact that the case when $x, y < 0$ is symmetric and gives the same integral.
Now, let $C = \{ (x, y, z) : x > 0, y > 0, x^2 + y^2 \leq 1, z \in \mathbb{R}^{d-2}, w^* \}$ be a quarter cylinder.
The volume of $C$ is $> 0$, and $C \subseteq S_{+}$.
Further, for all $x \in C$, we have:
\begin{equation} p(u) \Delta(w^*, w, u) + p(u') \Delta(w^*, w, u') > 0 \end{equation}
So applying Lemma~\ref{lem:basicExpectationBound} again, we get:
\begin{equation} L(w, P_{\mu, \sigma}) - L(w^*, P_{\mu, \sigma}) > 0 \end{equation}
Which completes the proof.
\end{proof}

With these key lemmas, the proof of Theorem~\ref{thm:gaussian} is straightforward.

\newtheorem*{gaussianTheorem}{Restatement of Theorem~\ref{thm:gaussian}}

\begin{gaussianTheorem}
\gaussianTheoremText{}
\end{gaussianTheorem}

\begin{proof}
The proof reduces to showing the one-step case: for $0 < t \leq T$, if $\|w_{t-1} - w^*(\mu_{t-1})\|_2 \leq \frac{1}{4}$ then $w_t = w^*(\mu_t)$, where the $w_t$ is selected according to Equation~\eqref{eqn:constrained_min}. Applying this one-step result inductively gives us the desired result, that $w_T = w^*(\mu_T)$.

For the one-step case, from Lemma~\ref{lem:lipschitzGaussian}, since $||\mu_{t-1}||_2, ||\mu_t||_2 \geq B > 0$, $||w^*(\mu_t) - w^*(\mu_{t-1})||_2 \leq \frac{1}{B}||\mu_t - \mu_{t-1}||_2 \leq \frac{1}{B}\frac{B}{4} = \frac{1}{4}$.
Then by triangle inequality, since $||w_{t-1} - w^*(\mu_{t-1})||_2 \leq \frac{1}{4}$, we have $||w_{t-1} - w^*(\mu_t)||_2 \leq \frac{1}{2}$.
Further, $||w^*(\mu_t)||_2 \leq 1$, and by Lemma~\ref{lem:localMinGaussian}, any other $w$ satisfying $||w_{t-1} - w||_2 \leq \frac{1}{2} < 1$, $||w||_2 \leq 1$, and $w \neq w^*(\mu_t)$ satisfies $U(w^*(\mu_t), P_{\mu_t, \sigma_t}) < U(w, P_{\mu_t, \sigma_t})$.
So $w^*(\mu_t)$ is the unique minimizer in the constrained set, which means $w_t = w^*(\mu_t)$.
\end{proof}

\newpage
\newpage
\section{Experimental details for Section~\ref{sec:experiments}}
\label{sec:appendix_experiments}

In this section, we provide additional experimental details, and give results for ablations for the experiments in Section~\ref{sec:doesGradualShiftHelpExperiments}.
An advantage of gradual self-training is that it has a very small number of hyperparameters and we show that our findings are robust to different choices of these parameters---even if we do not do confidence thresholding, train every method for more iterations, and use a smaller window size, gradual self-training does better than self-training directly to the target and the other baselines.
For reproducibility, we provide all code but we also describe our datasets and models here.

\subsection{Datasets}

We ran experiments on 3 datasets:
\begin{enumerate}
\item \emph{Gaussian in $d=100$ dimensions}: We randomly select an initial mean and covariance for each of the two classes, and a final mean and covariance for each class, all in $d$ dimensions. Note that unlike in the theory in Section~\ref{sec:gaussian_theory}, each class can have a different (non diagonal) covariance matrix. The initial and final covariance matrices can also be different. The marginal probability of each class is the same, 0.5. We get labeled data sampled from a gaussian with the initial mean and covariance. For the intermediate domains, we linearly interpolate the means and covariances for each class between the initial and final, and sample points from a gaussian with the corresponding mean and covariance matrices. The number of labeled and unlabeled samples is on the order of d (as opposed to exponential in d, which importance weighting approaches would need). We provide more details next.

Details: We sample $\mu_0^{(-1)}, \mu_0^{(+1)}, \mu_T^{(-1)}, \mu_T^{(+1)}$ independently from $\normal(0, I)$ in $d$ dimensions. Since $d$ is high, these are all nearly orthogonal to each other. We then sample covariance matrices $\Sigma_0^{(-1)}, \Sigma_0^{(+1)}, \Sigma_T^{(-1)}, \Sigma_T^{(+1)}$ independently by sampling a diagonal matrix and rotation matrix (since the covariance matrices are PSD they decompose into $UDU^{\top}$ for rotation matrix $U$ and diagonal matrix $D$). We first sample a diagonal matrix $D$ in $d$ dimensions where each entry is uniformly random and independently sampled between min\_var and max\_var. Then, we sample a rotation matrix $U$ from the Haar distribution (which is a standard way to sample random orthogonal matrices), and then compose these to get $UDU^{\top}$.

At all times, we keep $\prob(Y = +1) = \prob(Y = -1) = 0.5$. We now sample $N$ labeled examples from the source domain, where $\prob(X | Y=1) = \normal(\mu_0^{(+1)}, \Sigma_0^{(+1)})$ and $\prob(X | Y=-1) = \normal(\mu_0^{(-1)}, \Sigma_0^{(-1)})$.
We sample $T$ unlabeled intermediate examples. For $y \in \{-1, 1\}$, let $\mu_t^{(y)} = (t / T) \mu_0^{(y)} + ((T - t) / T) \mu_T^{(y)}$ and $\Sigma_t^{(y)} = (t / T) \Sigma_0^{(y)} + ((T - t) / T) \Sigma_T^{(y)}$. We then sample $y_t \sim \mbox{Bern}(0.5)$, and $x_t \sim \normal(\mu_t^{(y)}, \Sigma_t^{(y)})$---the model only gets to see $x_t$ but not $y_t$. The unseen target images are sampled from the final means and covariances for each class, and we measure accuracy on these held out examples.

We use $N$ = 500 (500 labeled examples from the source), $T$ = 5000 (so 5000 unlabeled examples in total), and use min\_var=0.05, max\_var=0.1 (the standard-deviation is the square root of these).  We sample 1000 target examples to check accuracy.

\item \emph{Rotating MNIST}: We split the training data, consisting of 50,000 images, using the first $N_{\textup{src}} = 5000$ images as the source training set, next $N_{\textup{val}} = 1000$ images as source validation set, next $N_{\textup{inter}} = 42000$ images as unlabeled intermediate examples, and the final $N_{\textup{trg}} = 2000$ images as unseen target examples. We rotate each source image by an angle uniformly selected between 0 and 5 degrees. The $i$-th intermediate example is rotated by angle $5 + 55i / N_{\textup{inter}}$ degrees. Each target image is rotated by an angle uniformly selected between 55 degrees and 60 degrees.

\item \emph{Portraits}: A more realistic dataset where we do not control the structure of the shift, consisting of photos of high school seniors taken across many years. Additionally, there is label shift, that is the proportions of males and females, $\prob(Y)$, changes over time (see Figure~\ref{fig:portraits_gender_ratios}), unlike our theory which assumes that the probability of each label stays constant. We use the first 2000 images as source images. We shuffle these, and use 1000 for training, and 1000 for validation. We use the next 14000 images as unlabeled intermediate examples. Finally, we use the next 2000 images as unseen target examples. We downsample the images to 32x32 but do no other preprocessing. We reserve images at the end of the dataset as held-out examples for future work, and so that we can test how the method extrapolates past the point we validate on.
\end{enumerate}

\begin{figure}[t]
\begin{center}
\ifdefined\usearxivstyle
\centerline{\includegraphics[width=0.5\columnwidth]{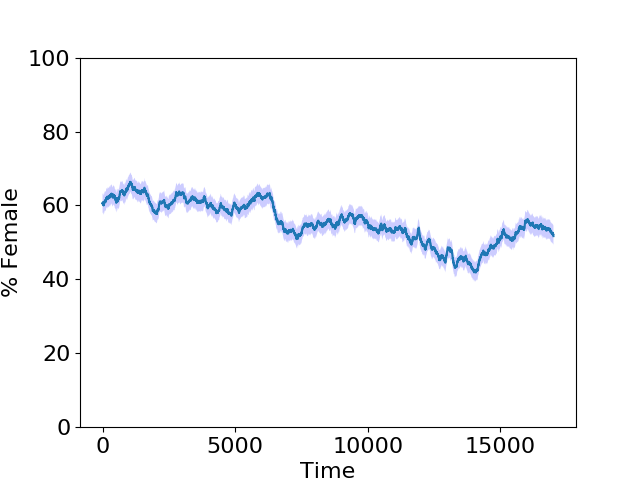}}
\else
\centerline{\includegraphics[width=\columnwidth]{images/portraits_gender_ratios.png}}
\fi
\caption{The plot shows a rolling average of the fraction of images that are female, over a window size of 1000, with 90\% confidence intervals. The plot suggests that the proportion of males and females changes over time, and is not constant---this label shift might make the task more challenging for self-training methods.}
\label{fig:portraits_gender_ratios}
\end{center}
\vskip -0.3in
\end{figure}

\subsection{Algorithm and baselines}

Next, we describe the gradual self-training algorithm and parameters in more detail.
Algorithm~\ref{alg:grad_self_training} shows pseudocode for gradual self-training.
filter\_low\_confidence filters out the $\alpha$ fraction of examples where the model is least confident, where confidence is measured as the maximum of the softmax output of the classifier.
This filtering is standard in many instances of self-training~\cite{xie2020selftraining}.

\begin{algorithm}[tb]
   \caption{Gradual Self-Training}
   \label{alg:grad_self_training}
\begin{algorithmic}
   \STATE {\bfseries Input:} Labeled source examples $S$, Intermediate unlabeled examples $I$, Window size $W$, Confidence threshold $\alpha \in (0, 1)$, Number of Epochs $n$, Regularized model $M$
   \STATE {\bfseries Assume:} $W$ divides $|I|$
   \STATE Train $M$ on $S$ for $n$ epochs
   \FOR{$t=1$ {\bfseries to} $|I|/W$}
   \STATE cur\_xs$ = I[(t-1)W:tW]$
   \STATE pseudolabeled\_ys$ = M$.predict\_labels$($cur\_xs$)$
   \STATE confident\_idxs$ = $filter\_low\_confidence$(M$, cur\_xs, $\alpha)$
   \STATE filtered\_xs = cur\_xs[confident\_idxs]
   \STATE filtered\_ys = pseudolabeled\_ys[confident\_idxs]
   \STATE Train $M$ on filtered\_xs, filtered\_ys for $n$ epochs
   \ENDFOR
\end{algorithmic}
\end{algorithm}

For the baselines---for target self-train, we self-train multiple times (iteratively) on the target. Each round of self-training uses the current model $M$ to pseudolabel examples in the target, and then trains on these pseudolabeled examples. Specifically, to make comparisons fair we self-train $|I| / W$ times on the target, so that the total number of self-training steps performed by the target self-train baseline and gradual self-training are the same. Similarly, when we self-train to all examples, we self-train multiple times on all unlabeled data, self-training $|I| / W$ times. Here $W$ is the window size in Algorithm~\ref{alg:grad_self_training} which is the number of examples in each intermediate domain.

Note that in the synthetic datasets (rotating MNIST and Gaussian) we ensure that target self-train gets access to the same number of unlabeled examples as gradual self-training does in total, to ensure that the improvements are not simply because gradual self-training consumes more unlabeled data (accumulated over all of the intermediate domains). For the real dataset (Portraits), we cannot generate additional examples for target self-train. However, this is why we also compare against self-training directly to all the unlabeled data, which gets access to exactly the same data that gradual self-training does but does not leverage the gradual structure.

\subsection{Models and parameter settings}

Next, we describe the models and parameter settings we used:
\begin{enumerate}

	\item \emph{Models}: For the Gaussian dataset we use a logistic regression classifier, with l2 regularization 0.02. For the MNIST and Portraits dataset, we use a 3 layer convolutional network. For each conv layer we use a filter size of 5x5, stride of 2x2, 32 output channels, and relu activation. We added dropout(0.5) after the final conv layer, and batchnorm after dropout. We flatten the final layer, and then apply a single linear layer to output logits (the number of logits is the number of classes in the dataset which is 10 for rotating MNIST and 2 for Portraits). We then take the softmax of the logits, and optimize the cross-entropy loss. We did not tune the model architecture for our experiments, however we checked that adding an extra layer, changing the number of output channels, and using a different architecture with an extra fully connected layer on top, have little impact on the results. 

	\item \emph{Parameters}: For the window size, we use $W = 500$ for the Gaussian dataset, and $W = 2000$ for the rotating MNIST and Portraits dataset. We use a smaller window for the Gaussian dataset because the data is lower dimensional and we have less unlabeled data. We train the model for 10 epochs, 20 epochs, and 100 epochs in each round for the rotating MNIST, Portraits, and Gaussian dataset respectively. These numbers were chosen on validation data on the source without examining the intermediate or target data, and we show an ablation which suggests that the results are not sensitive to these choices.

	\item \emph{Confidence thresholding}: We chose $\alpha = 0.1$ to filter out the 10\% least confident examples, since these are examples the model is not confident on, so the predicted label is less likely to be correct. We run an ablation without this filtering and see that all methods perform slightly worse, but the relative ordering is similar---gradual self-training is still significantly better than all the other methods.

\end{enumerate}

\subsection{Ablations}

We run ablations which suggest that the results in Section~\ref{sec:doesGradualShiftHelpExperiments} are robust to the choice of algorithm hyperparameters.

\textbf{Confidence thresholding}: Table~\ref{tab:confAblation} shows the results for rotating MNIST and Portraits without confidence thresholding. All methods do worse without confidence thresholding but gradual self-training does significantly better than the other methods.

\begin{table}[t]
\caption{
Classification accuracies for gradual self-train (ST) and baselines without confidence thresholding/filtering, with $90\%$ confidence intervals for the mean over 5 runs. All methods do worse without confidence thresholding but gradual self-training does significantly better than the other methods.
}
\label{tab:confAblation}
\vskip 0.15in
\begin{center}
\begin{small}
\begin{sc}
\begin{tabular}{lcccr}
\toprule
 & Rot MNIST & Portraits \\
\midrule
Source      & 30.5$\pm$1.0 & 76.2$\pm$0.5 \\
Target ST   & 31.1$\pm$1.4 & 76.9$\pm$1.3 \\
All ST      & 32.6$\pm$1.3 & 77.1$\pm$0.5 \\
Gradual ST  & \textbf{80.3$\pm$1.4} & \textbf{81.7$\pm$1.3} \\
\bottomrule
\end{tabular}
\end{sc}
\end{small}
\end{center}
\vskip -0.1in
\end{table}

\textbf{Window sizes}: Table~\ref{tab:smallerWindowAblation} shows the results for rotating MNIST and Portraits if we use smaller window sizes (from 2000 to 1000). Gradual self-training still does significantly better than the other methods.

\begin{table}[t]
\caption{
Classification accuracies for gradual self-train (ST) and baselines with smaller window sizes, with $90\%$ confidence intervals for the mean over 5 runs. Gradual self-training still does significantly better than the other methods.
}
\label{tab:smallerWindowAblation}
\vskip 0.15in
\begin{center}
\begin{small}
\begin{sc}
\begin{tabular}{lcccr}
\toprule
 & Rot MNIST & Portraits \\
\midrule
Source      & 35.6$\pm$1.7 & 74.1$\pm$1.4 \\
Target ST   & 36.0$\pm$1.5 & 77.9$\pm$1.4 \\
All ST      & 38.5$\pm$2.6 & 76.3$\pm$2.2 \\
Gradual ST  & \textbf{90.4$\pm$2.0} & \textbf{83.8$\pm$0.5} \\
\bottomrule
\end{tabular}
\end{sc}
\end{small}
\end{center}
\vskip -0.1in
\end{table}

\textbf{Additional ablations for Portraits}: We ran two additional ablations, focusing on Portraits. In the first ablation, we trained every method of self-training for 50\% more epochs. Over 5 trials, gradual self-training got an accuracy of $83.9 \pm 0.4\%$, target self-train got an accuracy of $80.7 \pm 1.1\%$, and self-training to all unlabeled examples got an accuracy of $79.6 \pm 2.2\%$. The non-adaptive baseline got an accuracy of $77.3 \pm 1.0 \%$.

We also ran an experiment on Portraits where we extrapolate further in time. Here we use the first 2000 images as source, next 20,000 images as unlabeled intermediate examples, and next 2000 images as the target. Here the accuracy of gradual self-training is $60.6 \pm 1.4\%$, self-training on the target directly is $56.5 \pm 1.4\%$, and self-training on all unlabeled data is $57.4 \pm 0.3\%$. Gradual self-training still does better, but all methods do quite poorly---developing and analyzing new techniques for gradual domain adaptation is an exciting avenue for future work.

\end{document}